\documentclass[sigconf, authorversion]{acmart}

\AtBeginDocument{%
  }

\setcopyright{acmlicensed}
\copyrightyear{2018}
\acmYear{2018}
\acmDOI{XXXXXXX.XXXXXXX}
\acmConference[Conference acronym 'XX]{Make sure to enter the correct
  conference title from your rights confirmation email}{June 03--05,
  2018}{Woodstock, NY}
\acmISBN{978-1-4503-XXXX-X/2018/06}

\usepackage{hyperref}

\usepackage{algorithm}
\usepackage{enumitem}
\newcommand{\solved}[1]{}
\usepackage{algpseudocode}
\usepackage{subfig}
\newcommand{\model}{\textsc{KRAFT}}
\newcommand{\GG}{\mathcal{G}}
\newcommand{\F}{\mathcal{F}}
\newcommand{\E}{\mathcal{E}}
\newcommand{\R}{\mathcal{R}}
\newcommand{\T}{\mathcal{T}}
\newcommand{\bb}[0]{\mathbf{b}}
\newcommand{\W}[0]{\mathbf{W}}
\newcommand{\Dr}[0]{\mathrm{Dr}}
\newcommand{\mlpmixer}{\textsc{MLP-Mixer}}

\newcommand{\LL}[1]{\textcolor{black}{#1}}

\usepackage{tikz}
\usepackage{soul}
\newcommand*\circled[1]{\tikz[baseline=(char.base)]{
            \node[shape=circle,draw,inner sep=0.3pt] (char) {#1};}}
\usepackage{tcolorbox}

\newtheorem{problem}{Problem}

\newtheorem{dfn}{Definition}

\newtheorem{lemma}{Lemma}

\newtheorem{example}{Example}

\newtheorem{prop}{Proposition}

\newcommand{\eat}[1]{}
\newcommand{\eatt}[1]{}
\newcommand{\eatfull}[1]{} 
\newcommand{\eatimport}[1]{}
\newcommand{\eatfinal}[1]{}

\copyrightyear{2025}
\acmYear{2025}
\setcopyright{cc}
\setcctype{by}
\acmConference[CIKM '25]{Proceedings of the 34th ACM International Conference on Information and Knowledge Management}{November 10--14, 2025}{Seoul, Republic of Korea}
\acmBooktitle{Proceedings of the 34th ACM International Conference on Information and Knowledge Management (CIKM '25), November 10--14, 2025, Seoul, Republic of Korea}\acmDOI{10.1145/3746252.3761291}
\acmISBN{979-8-4007-2040-6/2025/11}

\begin{document}

\title{KRAFT: A Knowledge Graph-Based Framework for Automated Map Conflation}

\author{Farnoosh Hashemi}
\email{sh2574@cornell.edu}
\orcid{0000-0002-4778-6608}
\affiliation{%
  \institution{Cornell University}
  \city{Ithaca}
  \country{USA}
}

\author{Laks V.S. Lakshmanan}
\email{laks@cs.ubc.ca}
\affiliation{%
  \institution{The University of British Columbia}
  \city{Vancouver}
  \country{Canada}
}

\renewcommand{\shortauthors}{Hashemi and Lakshmanan}

\begin{abstract}

Digital maps play a crucial role in various applications such as navigation, fleet management, and ride-sharing, necessitating their accuracy and \LL{currency, which require} timely updates. While the majority of geospatial databases (GDBs) provide high-quality information, their data is (i) limited to specific regions and/or (ii) missing some entities, even in their covered areas. Map conflation is the process of augmentation of a GDB using another GDB to conflate missing spatial features. Existing map conflation methods suffer from two main limitations: (1) They are designed for the conflation of linear objects (e.g., road networks) and cannot simply be extended to non-linear objects, thus missing information about most entities in the map. (2) They are heuristic algorithmic approaches that are based on pre-defined rules, unable to learn entities matching in a data-driven manner. To address these limitations, we design \model{}, a learning based approach consisting of three parts: (1) \textit{Knowledge Graph Construction} -- where each GDB is represented by a knowledge graph, (2) \textit{Map Matching} -- where we use a knowledge graph alignment method as well as a geospatial feature encoder to match entities in obtained knowledge graphs, and  (3) \textit{Map Merging} -- where we merge matched entities in the previous modules in a consistent manner, using a mixed integer linear programming formulation that fully merges the GDBs without adding any inconsistencies. Our experimental evaluation shows that not only does \model{} achieve outstanding performance compared to state-of-the-art and baseline methods in map conflation tasks, but each of its modules (e.g., Map Matching and Map Merging) also separately outperforms traditional matching and merging methods.

\end{abstract}

\eat{

}

\maketitle

\section{Introduction}\label{ch:Introduction}
Digital maps have become an indispensable part of our daily lives, revolutionizing the way we navigate, interact, search places of interest, and understand the world around us. On the other hand, in industrial uses, digital maps have an important role in navigation, transportation, route planning, fleet management, etc. Geospatial databases (GDBs) represent objects in a geometric space, such as vector data and raster data, and offer optimal manageability for analyzing them. The above applications require GDBs  with high-quality, wide coverage, and minimal missing spatial features while optimizing the cost and time for individual or industrial users. To address this need, several proprietary geospatial datasets as well as open-source GDBs have been developed. However, existing  GDBs are either \circled{1} limited to specific regions, \circled{2} missing some geospatial features/entities even in their covered areas, and/or \circled{3}  expensive to maintain and are not openly available.

While open-source GDBs do not suffer from the third drawback and allow ad hoc manual effort of adding missing features, the process is tedious, time-consuming, expensive, and is unlikely to scale to global coverage, resulting in having relatively poor coverage for many parts of the world \cite{thesis, gorisha,coverage1,coverage2}. \LL{To address this, \textit{map conflation} is defined as the task of merging two overlapping GDBs, combining accurate information from either database while minimizing the error}. Map conflation consists of two steps: \circled{1} \textit{Map Matching}: Given two GDBs $D$ and $D'$, pair all features $(E, E')$, where $E \in D$ and $E' \in D'$, such that $E$ and $E'$ represent the same physical entity on Earth in $D$ and $D'$, respectively. \circled{2} \textit{Map Merging}: \LL{Given a matched  pair of GDBs $D, D'$, we aim to augment the matched entities in $D$ with the unmatched entities from $D'$, in order to obtain a consistent map (i.e., GDB).} 
Geospatial databases are typically categorized as raster (image-based) or vector (geometry-based). Conflation, the process of merging two geospatial databases, can occur between vector-to-vector, vector-to-raster, or raster-to-raster types \cite{ruiz2011digital, Chen2008}. This study focuses on the vector-to-vector conflation problem.

Map conflation presents several challenges. First, spatial features often suffer from positional inaccuracies due to differences in data collection and projection methods, making direct location comparisons unreliable. Second, matching is not always 1:1, as entities may evolve over time, requiring one-to-many or many-to-one mappings across GDBs~\cite{gorisha}. Third, manual or human-supervised approaches are costly and lack scalability, limiting practical deployment~\cite{canavosio2015hootenanny, gorisha, thesis}.

To address these challenges in the matching process, several methods have been developed~\cite{walter1999matching, li2010optimized, tong2014linear, li2011optimisation, lei2019optimal,mustiere2008matching, tong2009probability, almotairi2018using, zhang2008delimited, schafers2014simmatching}, which can be divided into three main groups: \circled{i} Optimization-based methods: Formulate the matching problem \LL{as an optimization problem and solve for the optimal solution}~\cite{walter1999matching, li2010optimized, li2011optimisation}. While finding the optimal solution, these methods are are time-consuming and not scalable. \circled{ii} Greedy-based or heuristic methods: To address the scalability issue, heuristics are used to quickly find a sub-optimal solution, sometimes resulting in poor quality~\cite{gorisha, song2011relaxation, liu2015progressive}. \circled{iii} Rankjoin-based method: To mitigate the limitations of previous groups, Agarwal et al.~\cite{gorisha, thesis} proposed a matching framework by adapting the classic Rank Join in databases, where each edge of the source GDB is modeled as a relation, to obtain the best matching roads between a source and a target geospatial database.

Each of these methods only addresses a subset of the above challenges, and furthermore, they all  suffer from four main limitations: \circled{1} All these approaches are designed for and limited to linear objects (e.g., roads and sidewalks). However, real-world digital maps also consist of information about complex non-linear entities (e.g., buildings, lakes, open spaces, etc.), and these methods are not simply extendable to match the non-linear objects. \circled{2} These methods assume pre-defined patterns/rules for the discrepancy of GDBs and cannot \textit{learn} the discrepancy of two digital maps in a data-driven manner. \circled{3} Most existing methods represent entities in GDBs by some heuristics (e.g., their locations, centers, etc.), leading them to missed information, and knowledge about the relative position and neighborhood of entities in GDBs. \circled{4} Entities in GDBs are often associated with rich information or metadata (e.g., the name of the building/street, its functionality, etc.), which can greatly help the matching process. Existing methods cannot simply be extended to take advantage of this information when it is applicable.

On the other hand, map merging, surprisingly, has attracted less attention~\cite{gorisha, haunert2005link, zhang2016automatic, harrie1999constraint, touya2013conflation}. All existing methods suffer from the following limitations: ~\circled{1} Inconsistency: existing methods do not consider the inconsistency that might be caused by the merging process. That is, the resulting merged map can have abnormally intersecting roads or even intersecting roads and buildings. \circled{2} Large perturbations: These methods do not optimize the perturbations that are required to merge maps. Accordingly, the resulting maps can have entities with a large perturbation w.r.t. their actual location, changing the structure and accuracy of the map.

To address the aforementioned drawbacks, in this paper, we introduce \model{}, a general machine-learning-based framework for map conflation. Given two GDBs $D$ and $D'$, in the first step, \model{} constructs a knowledge graph for each GDB, which represents entities as nodes and their relative positions as edges. In the second step, to match entities in a scalable and data-driven manner, \model{} uses a knowledge graph alignment method, which incorporates the information from both the relative positional relations and entities' metadata to match entities in the GDBs. This alignment method first employs multiple Graph Neural Networks (GNNs) to aggregate the neighborhood information within multiple hops and then uses an attention mechanism to learn important neighbors in an end-to-end manner. After matching entities, to merge the two GDBs, we formulate the problem as a  mixed integer linear programming problem to avoid overlaps of entities as well as minimize their perturbations. We make the following contributions:
\begin{enumerate}[topsep=0pt]
    \item We introduce a new method to represent geospatial databases as knowledge graphs and develop an algorithm to efficiently construct the knowledge graph representation of the GDBs. 
    \item We design a scalable end-to-end learning-based method to find the best matching between nodes in the two knowledge graphs. Considering the unique properties of digital maps, this method not only uses the relative positional features of the entities in multiple hops but it also can use the rich metadata information about entities for matching objects, whenever it is available. 
    \item We formulate the map merging task as a mixed  integer linear programming problem, where we limit the movement of each entity in each direction by a parameter $\epsilon > 0$. We next minimize the movements of entities while avoiding overlaps \LL{and} guaranteeing a consistent merged map. 
    \item Our extensive experimental evaluations on two open-source GDBs, OpenStreetMap (OSM) and Boston Open Data (BOD), and in different tasks and settings (map matching, merging, and conflation of linear/non-linear objects) show high precision and recall, and superior performance of \model{} compared to baselines.
\end{enumerate}
Additional details and analyses appear in the appendix.

\vspace{-1ex}
\section{Related Work}
\label{ch:RW-short}

Several methods have been proposed for map matching between geospatial databases, typically categorized into greedy heuristics, optimization-based approaches, and, more recently, rank-join methods. Greedy methods~\cite{zhang2008delimited, mustiere2008matching, schafers2014simmatching} iteratively match the most similar entities based on factors such as Hausdorff distance, angle, orientation, and topological connectivity. Algorithms like DSO and NetMatcher use hand-crafted thresholds to select matches, while SimMatching uses a similarity matrix updated by local topological consistency. Although these approaches are computationally efficient, they often assume unrealistic conditions (e.g., that one GDB is a subset of another), are tailored to specific object types (e.g., only roads), and rely on manually specified similarity measures, limiting their generalizability and robustness.

Optimization-based methods~\cite{walter1999matching, li2010optimized, li2011optimisation, tong2014linear, lei2019optimal} aim to find globally optimal matches. Early work modeled map matching as maximizing mutual information~\cite{walter1999matching}, while others formulated it as an assignment problem minimizing Hausdorff distance~\cite{li2010optimized} or employed iterative logistic regression~\cite{tong2014linear}. Some, like~\cite{lei2019optimal}, framed matching as a $p$-matching flow problem. While these methods better capture global consistency, they often ignore topological information, require significant meta-information, or become computationally expensive for large-scale GDBs.

Most relevant to our work, \citet{gorisha} proposed \emph{MAYUR}, which adapts the Pull Bound Rank Join algorithm~\cite{schnaitter2008evaluating} to match edges between databases efficiently. MAYUR represents a key advance by drawing from relational database techniques, but remains limited to matching linear features and relies on predefined similarity metrics.  \citet{new-rw-ali} instead developed a scalable HMM-based conflation pipeline between LRS basemaps and OSM, but their focus remains restricted to linear road networks.

Beyond matching, map merging methods like  Rubbersheeting~\cite{haunert2005link, gorisha, chen2006automatically, song2008automated, zhang2016automatic, katzil2005spatial} attempt to integrate unmatched entities. It models one map as a flexible membrane distorted to align with another using control points — matched entities from prior map matching — as anchors. Unmatched features are deformed based on interpolated displacement vectors between control points. Different methods vary in how they interpolate shift vectors; for instance,~\citet{haunert2005link} assume pre-defined links between spatial features, while~\citet{zhang2016automatic} derive displacement vectors via Delimited Stroke Oriented (DSO) matching~\cite{zhang2008delimited}. Other approaches incorporate geometric constraints to minimize Euclidean distances between corresponding features in the source and conflated databases after merging~\cite{harrie1999constraint, touya2013conflation}. 
Despite their popularity, rubbersheeting-based methods cannot guarantee consistency; overlapping or misaligned objects often remain after merging. Recently, \citet{gorisha} introduced a weighted averaging of shift vectors for matched terminal points, assigning weights inversely proportional to distances, but challenges of ensuring a fully consistent merged map persist. Moreover, \citet{new-rw-driving} introduced FlexMap Fusion, a modular HD-map conflation pipeline that aligns LiDAR-based vector maps with OSM through georeferencing and semantic enrichment for autonomous driving.

Recent advances in graph machine learning have motivated modeling road networks as graphs~\cite{ETA, gorisha, thesis, ETAsigmodHistorical, ETAGAN, ETAtrajectory, ETAstochasticRecurrent}. Learning road representations has led to state-of-the-art results in travel-time prediction, with early work adapting convolutional networks for spatial data~\cite{ETAtrajectory} and later approaches employing recurrent and GAN-based models~\cite{ETAstochasticRecurrent, ETAGAN}. However, existing methods typically focus only on linear features like roads and ignore the broader relational structure of digital maps, limiting their expressiveness.\solved{ In sections~\ref{ch:KG} and \ref{ch:map-matching}, we propose a unified framework that represents all map entities as a knowledge graph.  }
A more detailed review of related work can be found in Appendix~\ref{ch:RW}.

\label{ch:PN}
\section{Preliminaries and Notation}
We first present the  terminology and background concepts we use on geospatial databases (GDBs) and  knowledge graphs. 

\subsection{Terminology}\label{sec:terminology}

\begin{dfn}[Terminal Point]
{\rm
\textit{Terminal} points are either dead-ends or intersections of roads. A point is classified as terminal if it has: a single road line, more than two road lines, or two road lines diverging at an angle greater than \LL{a threshold} $\theta$. \qed} 
\end{dfn}
Conversely, all other points found within linear roads are considered \textit{intermediate} points.
The range of the angle between two segments is $[0, \theta_0]$, where $\theta_0 = 180^{\circ}$ or $\theta_0 = \pi \:\:\mathrm{rad}$. In our experiments, we consider the threshold $\theta$ as $\theta = 45^{\circ}$ or $\theta = \frac{\pi}{4} \:\:\mathrm{rad}$.

\begin{dfn}[Segment]~\label{dfn:segment}
{\rm A segment \(s = (p, n_1, n_2, \ldots, n_m, p\prime)\) consists of an ordered sequence of intermediate points $n_i$ between two terminal points $p, p\prime$, each connected by a straight line. These points define the road segment's shape and orientation. \qed} 
\end{dfn}
While segments are used to represent linear objects (e.g., roads and buildings) in the map, next, we define polygons, which are used to represent non-linear objects in digital maps. 

\begin{dfn}[Polygon]~\label{dfn:polygon}
{\rm A polygon is a 2D shape that is formed by a closed loop of at least three straight, interconnected sides.\qed} 
\end{dfn}

\begin{dfn}[Geospatial Database]
    {\rm A \textit{geospatial database} (GDB for short) is  a triple $G = (E, SE, F)$, where $E$ is the set of all non-linear entities in the map, $SE$ is the set of segments, and $F(e)$ represents the features of entity $e \in E\cup SE$. \qed} 
\end{dfn}
In GDBs, the provided features depend on the source of the data. Most of the time, these features include: the type of the object (e.g., building, sidewalks, etc.), the location of its points or centers, or other general metadata (e.g., address, year of construction, etc.).

As discussed in \autoref{ch:Introduction}, locations of objects in GDBs might be noisy. Next, we discuss circular error which is used to model this uncertainty in the data. 

\textbf{Circular Error.}
\textit{Circular Error} models the positional uncertainty of spatial features in a geospatial database, representing the maximum deviation \LL{between the true and recorded}   coordinates of points.

\begin{dfn}[Homologous Objects]
    {\rm \textit{Homologous} spatial objects are an ordered pair of spatial objects, each from a different GDB, representing the same physical entity on Earth. \qed} 
\end{dfn}

\begin{dfn}[Knowledge Graph]
    {\rm A knowledge graph $\GG$ is quadruple $\GG = (\E, \R, \T, \F)$, where $\E$ is the set of entities, $\R$ is the set of relations, $\T \subseteq \E \times \R \times \E$ is the set of triples \LL{connecting pairs of entities via relations}, and $\F: \E \rightarrow \mathbb{R}^{K}$ represents the contextual features of entities. \qed} 
\end{dfn}

\subsection{MLP-Mixer}\label{sec:mlp-mixer}
Comparing to modern architectures~\citep{transformer, behrouz2024titans, behrouz2025miras, behrouz2025atlas, dai2019transformer}, \textsc{MLP-Mixer} \cite{mlp-mixer} proposes a simple and efficient model family using multi-layer perceptrons (MLPs). The architecture has two main components: (1) Token Mixer, learning cross-feature dependencies, and (2) Channel Mixer, ensuring positional invariance akin to convolutions. The roles of the Token and Channel Mixer phases are outlined next:

\textbf{Token Mixer.}
Let $\mathbf{E}$ be the input of the \textsc{MLP-Mixer}, then the token mixer phase is defined as:
\begin{equation}
     \mathbf{H}_{\text{token}} = \mathbf{E} +  \mathbf{W}_{\text{token}}^{(2)} \sigma\left(   \mathbf{W}_{\text{token}}^{(1)} \texttt{LayerNorm}\left(  \mathbf{E}\right)^T  \right)^T,
\end{equation}
where $\sigma(.)$ is a nonlinear activation function (usually GeLU~\cite{gelu}). Since it feeds the input's columns to an MLP, it mixes the cross-feature information.

\textbf{Channel Mixer}.
Let  $\mathbf{H}_{\text{token}}$ be the output of the Token Mixer, then the Channel Mixer phase is defined as:
\begin{equation*}
     \mathbf{H}_{\text{channel}} = \mathbf{H}_{\text{token}} +  \mathbf{W}_{\text{channel}}^{(2)} \sigma\left(   \mathbf{W}_{\text{channel}}^{(1)} \texttt{LayerNorm}\left(  \mathbf{H}_{\text{token}}\right)  \right),
\end{equation*}
where $\sigma(.)$ is a nonlinear activation function (usually GeLU~\cite{gelu}).

\subsection{Problem Statement}\label{sec:problem-statement}
In map conflation, we have a source database and a target database. Following previous work~\cite{canavosio2015hootenanny, thesis, gorisha}, we assume that the source database is immutable, meaning that the spatial properties of its geometric features cannot be changed. The conflated database is created by aligning entities from the target to the source and merging unmatched entities and non-spatial attributes from the target into a copy of the source. This approach assumes the source offers broader coverage, while the target provides higher-quality data for specific regions~\cite{thesis}.

For map matching, we assume that for each entity in the map $e \in E$, we have a vector $h_e \in \mathbb{R}^{d}$ that encodes both its positional and contextual features. We use $H_E$ to denote the matrix whose rows are the encodings of entities in $E$. In Section~\ref{ch:map-matching} we discuss how we obtain these encodings. 

\begin{dfn}[Map Matching]
{\rm Given a source GDB $G_s = (E_s, SE_s, F_s)$ with its entity encodings $H_{E_s}$, a target GDB $G_t = (E_t, SE_t, F_t)$ with its entity encodings $H_{E_t}$, and a similarity function $\textsc{Sim}: (E_s \cup SE_s) \times (E_t \cup SE_t) \rightarrow \mathbb{R}^{\geq 0}$, we aim to find a \textit{matching} defined as:
$$\mathcal{M} = \left\{ (e_i^s, e_j^t) \;\middle|\; e_i^s \in E_s \cup SE_s,\; e_j^t \in E_t \cup SE_t,\; \textsc{Sim}(h_i^s, h_j^t) \geq \mathscr{S} \right\}$$
such that $\sum_{(e_j^t, e_i^s) \in \mathcal{M}} \textsc{Sim}(h_j^t, h_i^s)$ is maximized over all possible matchings. \qed}
\end{dfn}

Next, given the definition of the map matching problem, we formally define the problem of map merging as follows:

\begin{dfn}[Map Merging]
    {\rm Given an immutable source database $G_s = (E_s, SE_s, F_s)$, a target database $G_t = (E_t, SE_t, F_t)$, and a matching set $\mathcal{M}$ between entities of $G_s$ and $G_t$, the merged geospatial database of $G_s$ and $G_t$ is a conflated geospatial database $G^*$ obtained as follows. 
 (1) Every entity $e$ in  $\{E_s\cup SE_s\}$ is retained in $G^*$. (2) Let $\Delta \E \subset \{E_t  \cup SE_t \}$ be the set of target entities not covered by the matching set $\mathcal{M}$. For every entity $e \in \Delta \E$ there is a corresponding entity $e^* \triangleq \Omega(e) \in G^*$ where $\Omega(e)$ shows the new position of entity $e$ and updates the spatial coordinates of all its points. (3) Merging entities $e \in \Delta \E$ in their new position $\Omega(e)$ results in no inconsistencies in the geospatial database $G^*$, while \LL{the sum of perturbations of entities in $e^* \in G^*$ w.r.t. the original position of the corresponding entity $e \in G_t$ is minimized.} 
 \qed} 
 \end{dfn}

\begin{dfn}[Map Conflation]
    {\rm The \textit{map conflation} process consists of two primary steps: (a) map matching and (b) map merging. The map matching step identifies homologous entities between the map databases. Subsequently, during map merging, any additional unmatched entities from the target database are merged with the entities present in the source database. \qed} 
\end{dfn}

\section{Method and Framework}\label{sec:framework}

\subsection{Overview of the Framework}
Our proposed framework, KRAFT, consists of three main modules. First, we construct a knowledge graph representation for each geospatial database, where entities are nodes and spatial relationships are modeled as edges. This construction captures both the relative positional information and metadata of linear and non-linear spatial features. Next, we perform homologous entity matching by aligning the knowledge graphs, leveraging multi-hop graph neural network encoders and attention-based aggregation to robustly match entities under spatial noise and representation shifts. Finally, we identify unmatched entities and formulate the map merging task as a mixed integer linear programming (MILP) problem, allowing us to consistently integrate unmatched objects into the source map while minimizing spatial perturbations and avoiding inconsistencies. Figure~\ref{fig:framework} illustrates the overall architecture of the KRAFT framework. In the following sections, we provide a detailed description of each component.

\begin{figure*}[!ht]
    \centering
     \vspace{-2ex}
\includegraphics[width=0.7\linewidth]{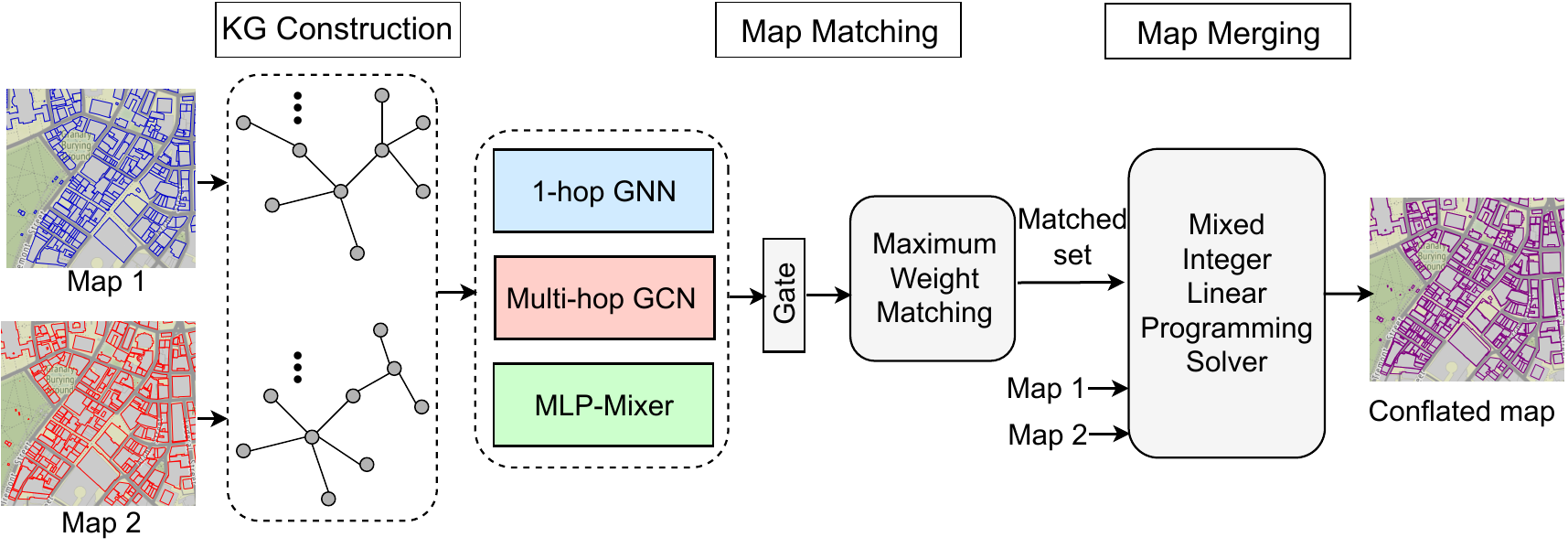}
 \vspace{-2ex}
    \caption{\model{} framework consists of four stages: (1) Knowledge graph construction, (2) Learning positional and contextual features (3) Matching via Maximum Weight Matching, and (4) Merging via Mixed Integer Linear Programming Solver. }
    \label{fig:framework}
\end{figure*}

\subsection{Knowledge Graph Construction} \label{ch:KG}
As discussed in Section~\ref{ch:Introduction}, a key limitation of existing map conflation methods is their reliance on heuristics to represent GDB entities (e.g., locations, centers, corners). Directly matching entities within GDBs, however, demands time-consuming algorithms ~\cite{gorisha, thesis}. To address this, we propose representing each GDB as a knowledge graph, enabling the use of deep learning to encode entities while preserving all their information.

 One of the main limitations of heuristics is their failure to capture an entity's neighborhood, which can be crucial for matching. To address this, given a GDB \( G = (E, SE, F) \), we center a \( K \times K \) grid at each non-linear entity \( e_n \in E \). While our method supports any \( K \), we use \( K = 3 \) for simplicity. With nine grid squares, we define nine relation types: \circled{1} Bottom, \circled{2} Bottom-Right, \circled{3} Right, \circled{4} Top-Right, \circled{5} Top, \circled{6} Top-Left, \circled{7} Left, \circled{8} Bottom-Left, and \circled{9} Close.

For each linear object (e.g., roads, sidewalks) \( e_l \in SE \), we define a buffer of width \( \lambda > 0 \) around \( e_l \), conforming to its shape as determined by connected segments. We then define two relation types: \circled{i} ``Inside'' if a non-linear entity is within the buffer, and \circled{ii} ``Connected'' if a linear entity shares an endpoint with \( e_l \).
We construct a knowledge graph $\GG = (\E, \R, \T, \F)$, where $\E = E \cup SE$ is the set of entities in the GDB and $\R$ is the set of the above relations. For each non-linear entity $e_n \in E$, if an entity $e' \in E \cup SE$ is in one of the squares of its grid, we add a triple $(e_n, r, e')$ to $\T$, where $r \in \R$ is their corresponding relation. For each linear entity $e_l \in SE$ we add triple $(e_l, \text{``Inside''}, e')$ (resp. $(e_l, \text{``Connected''}, e')$) to $\T$ if a non-linear entity $e' \in E$ is inside the $e_l$'s buffer (resp. if a linear entity $e' \in SE$ has an endpoint within distance $\delta$ from the endpoints of
$e_l$).  Algorithm~\ref{alg:KG-construction} in Appendix~\ref{sec:appendix_algorithms} describes this process.

\begin{figure}
\centering
{{\includegraphics[width=0.35\textwidth]{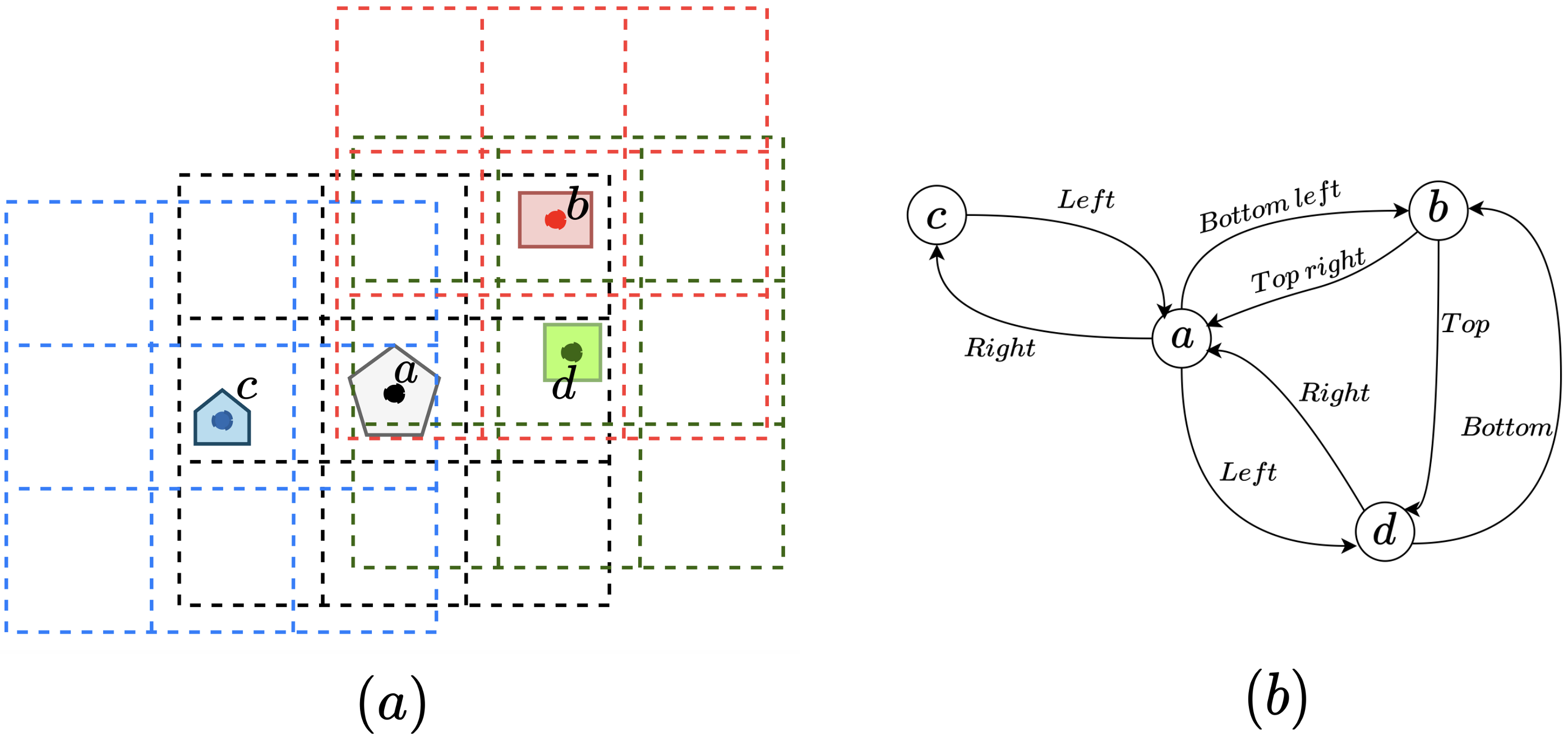} }}

    \vspace{-2ex}
    \caption{(a) Four polygons $a, b, c,$ and $d$ with local grids centered at each polygon's center. (b) Corresponding nodes in the knowledge graph, each connected to others within its grid using relative relation types.}
    \label{fig:build_kg}
\end{figure}

\begin{example}
Figure~\ref{fig:build_kg} (a) shows four polygons ($a$–$d$), each represented as a node in the knowledge graph. A grid of width $\mu$ is generated around each polygon’s center to define its neighborhood. In $a$'s grid, $b$, $c$, and $d$ appear at the top right, top left, and right. In $b$'s grid, $a$ and $d$ are at the bottom left and bottom. For $c$, only $a$ appears (right); and for $d$, both $a$ and $b$ appear (left and top). Figure~\ref{fig:build_kg} (b) shows the resulting knowledge graph.

\end{example}

\begin{example}
Figure~\ref{fig:linear-objects}(a) shows two non-linear entities ($b_1$, $b_2$) and three segments ($S_1$, $S_2$, $S_3$). Knowledge graph nodes are created for each, with a buffer of $\lambda$ around segments. Since $b_1$ and $b_2$ lie within $S_1$’s buffer, they connect to $S_1$. $S_2$ and $S_3$ connect to $S_1$ via "connected" relations, as their endpoints are within $\delta$ of $S_1$’s endpoint $p_2$.

\end{example}

\begin{figure}
\centering
{{\includegraphics[width=0.35\textwidth]{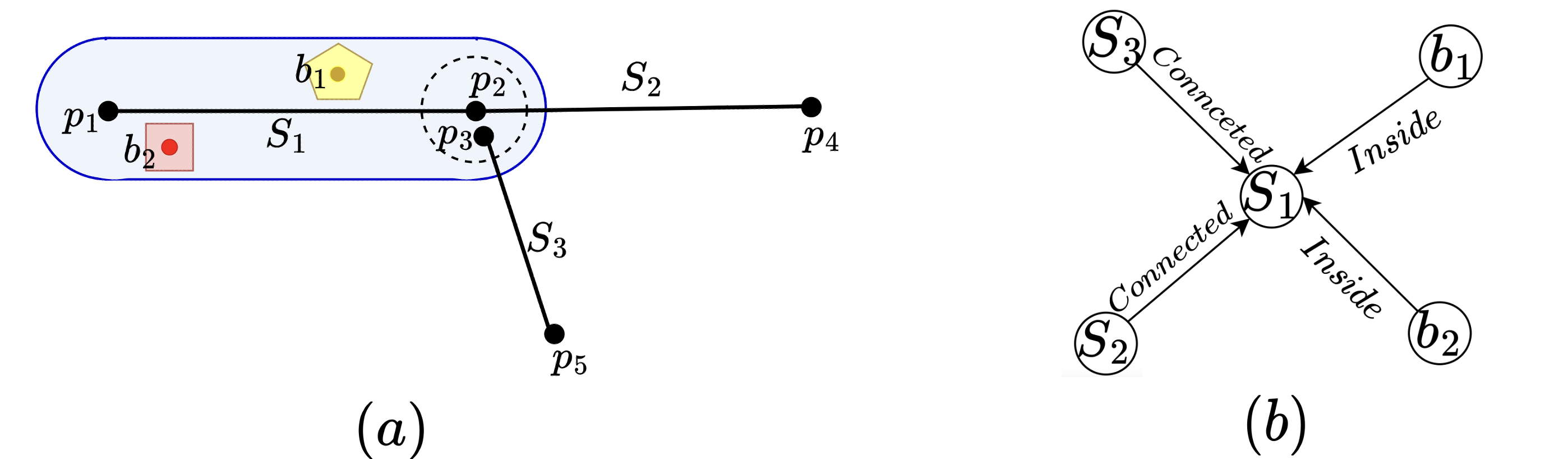} }}

    \vspace{-2ex}
    \caption {(a) Neighborhood around segment $S_1 = (p_1, p_2)$ in a geospatial database. (b) Adjacent nodes to $S_1$'s node in the knowledge graph with relation types: polygons $b_1$ and $b_2$ lie within $S_1$'s buffer and connect via "Inside." Segment $S_2 = (p_2, p_4)$ shares endpoint $p_2$ with $S_1$, and $S_3 = (p_3, p_5)$ has $p_3$ within $\delta$ of $p_2$; both connect to $S_1$ via "Connected."}
    \label{fig:linear-objects}
\end{figure}



\textbf{Time Complexity.} Using R-Trees~\cite{R-tree} to index entities in the GDB, we can  search each entity or segment in $\mathcal{O}(\log |E \cup SE|)$ time on average. 
Since for each entity (resp. segment) we explore its neighborhood within a grid (resp. buffer), we need $\mathcal{O}(|E \cup SE|$ $\log(|E \cup SE|) \times N_{\max})$ operations, where $N_{\max}$ is the maximum number of neighbors within a grid of with width $\mu$. In practice, we find that $N_{\max} \ll \frac{|E \cup SE|}{\log(|E \cup SE|)}$, showing the efficiency of this algorithm compared to the naive implementation.

\subsection{Map Matching}
\label{ch:map-matching}
Using  Algorithm~\ref{alg:KG-construction}, we first construct knowledge graphs $\GG_s = (\E_s, \R_s, \T_s, \F_s)$ and $\GG_t = (\E_t, \R_t, \T_t, \F_t)$ to represent the source and target GDBs respectively. Next, we encode each entity in the knowledge graphs by learning both its positional and contextual features. Finally, we use the obtained encodings to match the entities of source and target GDBs. To learn both the positional and contextual features, we use three modules: \circled{1} 1-hop Encoding, \circled{2} Multi-hop Encoding,
 and \circled{3} Feature Encoding. The first two modules are responsible to learn the structural properties while Feature Encoder is used to learn contextual features. 

As discussed in Section ~\ref{ch:Introduction}, the location of entities in the source and target GDBs might be noisy. This noise can even affect the process of knowledge graph construction.  
Figure~\ref{fig:2hop} in Appendix~\ref{sec:matchingappend} illustrates even how small noise can break direct connections in the knowledge graph but relationships can be preserved through 2-hop neighborhoods, thus aiding alignment.  

Expanding on this, leveraging $k$-hop neighborhoods ($k \geq 2$) helps mitigate location noise in GDBs. However, two key challenges arise: (1) The 1-hop neighborhood often provides the most crucial information~\cite{alinet}, and treating all hops equally may degrade performance. (2) Entities in higher-order neighborhoods contribute differently to an entity’s encoding, requiring an adaptive approach to integrate their information effectively.

\noindent
To this end, we need to learn which entities in the higher-order neighborhoods are more important than others. However, using a single $k$-layered graph neural network, and recursively propagating the information from $k$-hop neighborhoods is inefficient and would treat all entities alike, missing the differential importance of connections in higher-order neighborhoods~\cite{alinet}.   
To address the challenges, we separate the 1-hop and multi-hop encoding processes and use a learnable gate to combine their information. This gate learns when the 1-hop neighborhood is more informative than multi-hop neighborhoods, addressing Challenge (1). We also use a GNN to encode 1-hop neighbors and a modified Graph Attention Network (GAT)~\cite{GAT} to encode higher-order neighborhoods, simultaneously learning the importance of entities in these neighborhoods, thus addressing Challenge (2).

\noindent
\textbf{Knowledge Graph 1-hop Encoding.}
Graph neural networks effectively learn local graph structures through message passing~\cite{gnnbook, graphmamba}. To capture 1-hop information, we apply layer-wise GNN propagation with self-feature modeling, defined as:

\begin{align}\nonumber
    h_e^{(\ell + 1)} = \Dr\bigg\{&\sigma \bigg(h_e^{(\ell)} {\W_{\text{s}}}^{(\ell + 1)} + \sum_{e' \in N(e)} [\frac{h_{e'}^{(\ell)}}{\sqrt{p_{e'} p_e}} {\W}^{(\ell + 1)} + {\bb}^{(\ell + 1)}] \bigg)\bigg\},
\end{align}
where ${h_e}^{(\ell + 1)} \in \mathbb{R}^{d^{(\ell + 1)}}$ is the learned new features of entity $e$ in the $(\ell + 1)$-th GCN layer, ${h_e}^{(\ell)} \in \mathbb{R}^{d^{(\ell)}}$ is the hidden feature of $e$ in $\ell$-th GNN layer, and ${\W_{\text{s}}}^{(\ell + 1)}, {\W}^{(\ell + 1)} \in \mathbb{R}^{d^{(\ell)} \times d^{(\ell + 1)}}$, and ${\bb}^{(\ell + 1)} \in \mathbb{R}^{d^{(\ell + 1)}}$ are trainable weights. $\sigma(.)$ is a nonlinearity, e.g., ReLU, and $\Dr(.)$ is the dropout method~\cite{dropout} \LL{used} to avoid overfitting. Given an entity $e \in \E$ and its neighbor $e' \in N(e)$, we use the term $\frac{1}{\sqrt{p_e p_{e'}}}$  to normalize messages coming from the neighborhood of $e$, where $p_e = |N(e)| + 1$.  The input feature of entity $e$ in the first layer, $h_e^{(0)} \in \mathbb{R}^d$, is the normalized feature vector $\F(e)$.

\noindent
\textbf{Knowledge Graph Multi-hop Encoding.}
As we discussed earlier, multi-hop neighborhoods of each node provide us with complementary information that can mitigate the noise in the location of entities in the GDBs.  However, entities in higher-order neighborhoods have different importance in multi-hop encoding. To address this, we use a modified Graph Attention Network (GAT)~\cite{GAT}, which learns both neighborhood structures and the varying importance of neighbors, enabling weighted message aggregation.

We define \( N_k(e) \) as the set of entities at a distance of exactly \( k \) from \( e \).
For the sake of simplicity, we use $k = 2$ throughout the rest of this study. However, all equations and processes can simply be extended to any $k \geq 2$. Next, we use the following propagation rule:
\begin{align}\label{eq:multi-hop-propagation}\nonumber
    \psi_e^{(\ell + 1)}\!\!= \Dr\bigg\{&\sigma \bigg( {\W_k}^{(\ell + 1)} \psi_e^{(\ell)} \!\!+ \!\!\hspace{-2ex}\sum_{e' \in N_k(e)}\!\!\!\! [\frac{ \alpha_{e, e'} \times \psi_{e'}^{(\ell)}}{\sqrt{p^k_{e'} p^k_e}} {\W_k}^{(\ell + 1)} \!\!+ {\bb_k}^{(\ell + 1)}] \bigg)\bigg\},
\end{align}
where ${\W_k}^{(\ell + 1)}$ and ${\bb_k}^{(\ell + 1)}$ are learnable parameters, $p^k_e$ denotes its $k$-th degree plus one, i.e., $p^k_e = |\mathcal{N}_k(e)| + 1$, and $\alpha_{e, e'}$ is a learnable attention weight for entities $e$ and $e'$. Again, term $\frac{1}{\sqrt{p_{e'}^k p_e^k}}$ is used to normalize sent messages and avoid inconsistency of the scale of the encodings. 
The input feature of entity \( e \) in the first layer, \( \psi_e^{(0)} \in \mathbb{R}^d \), is the normalized feature vector \( \F(e) \).

To learn attention weights, one approach is to use GAT~\cite{GAT} as is, applying a shared linear transformation to all entities. However, in knowledge graphs, central entities and neighbors differ due to varying node types and connections. Following Sun et al.~\cite{alinet}, we use separate linear transformations for them and define $\alpha_{e, e'}$ as:
\begin{equation}\label{eq:attn}
    \alpha_{e, e'} = \texttt{Softmax}\left( \sigma\left(  [\mathbf{M}_c \psi^{(\ell)}_e]^T [\mathbf{M} \psi^{(\ell)}_{e'}] \right) \right),
\end{equation}
where $\mathbf{M}_c$ and $\mathbf{M}$ are learnable matrices that we use to transform the central entity and its neighbors.

\textbf{Feature Encoding.}
 GDBs often provide rich but potentially noisy meta information beneficial for conflation, such as building functionality or name. We use a separate module to encode these contextual features while accounting for their possible noise or misleading nature. To this end, we employ a gate module that learns to integrate structural and contextual features in an end-to-end, data-driven manner.

In geospatial data, the dependencies across features are also important. For example, given the location of entities as their contextual features $(x_i, y_i)$, cross-sample dependencies help encode both $x$ and $y$ values, while cross-feature dependencies capture map scale. We use an \mlpmixer{}~\cite{mlp-mixer} module to encode contextual features by learning both cross-data sample and cross-feature dependencies. 
\mlpmixer{}~\cite{mlp-mixer} is a family of simple, efficient models based on multi-layer perceptrons, commonly used in computer vision. It leverages both cross-feature and cross-data sample dependencies through token and channel mixers. In the token mixer, a 2-layer MLP encodes the feature of each entity, while the channel mixer transposes the output of the token mixer and uses a 2-layer MLP to capture cross-feature dependencies. Therefore, the channel mixer is defined as:

\begin{equation}\label{eq:channel-mixer}
     \Phi_{\text{out}} = \texttt{MLP} \left(\mathbf{H}_{\text{token}} +  \sigma\left(   \texttt{LN}\left(  \mathbf{H}_{\text{token}}\right) \mathbf{W}_{\text{channel}}^{(1)} \right) \mathbf{W}_{\text{channel}}^{(2)}\right) ,
\end{equation}
where \begin{equation}\label{eq:token-mixer}
     \mathbf{H}_{\text{token}} = \mathbf{F} +  \mathbf{W}_{\text{token}}^{(2)} \sigma\left(   \mathbf{W}_{\text{token}}^{(1)} \texttt{LN}\left(  \mathbf{F}\right)  \right).
\end{equation}
Here, $\mathbf{W}_{\text{token}}^{(1)}$, $\mathbf{W}_{\text{token}}^{(2)}$, $\mathbf{W}_{\text{channel}}^{(1)}$ and $\mathbf{W}_{\text{channel}}^{(2)}$  are learnable parameters,
$\sigma(.)$ is the activation function, and we use GeLU~\cite{gelu} following the MLP-Mixer \cite{mlp-mixer}. The layer normalization is used as it enables smoother gradients, faster training, and better generalization accuracy~\cite{layer-norm1}. 
In Equations~\ref{eq:channel-mixer} and~\ref{eq:token-mixer}, we use a residual network with skip connections for identity mapping to mitigate vanishing gradients. The $\texttt{MLP}(.)$ in Equation~\ref{eq:channel-mixer} adjusts the output to an arbitrary dimension. We represent matrix $\Phi$ by its rows, $\Phi = \begin{pmatrix} \varphi_1, \varphi_2, \dots, \varphi_{|\E|} \end{pmatrix}^T$, and use $\varphi_e$ for the feature encoding of entity $e \in \E$.

\noindent
\textbf{Mixer Gate.}
We first encode the 1-hop neighborhood, multi-hop neighborhood, and entity features separately to capture structural and contextual properties. To integrate these encodings, we aggregate them into a final entity representation. However, their importance varies, and features may be unavailable or noisy. Inspired by Sun et al.~\cite{alinet} and skip connections in neural networks~\cite{skip-connection}, we use a learnable gate to combine information from the 1-hop, multi-hop, and feature encoders.
Given the hidden representations $h_e^{(\ell)}$, $\psi_e^{(\ell)}$, and $\varphi_e$, we have:
\begin{equation}\label{eq:gate}
    \Upsilon_e^{(\ell)} = \zeta_e h_e^{(\ell)} + \eta_e \psi_e^{(\ell)} + (1 - (\frac{\zeta_e + \eta_e}{2})) \varphi_e,
\end{equation}
where $\zeta_e$ and $\eta_e$ are defined as: $\zeta_e = \sigma\left(  \mathbf{W}^{1}_{\text{gate}} \psi_e^{(\ell)}  +  \mathbf{W}^{2}_{\text{gate}} \varphi_e + \mathbf{b}_{\text{gate}}  \right)$ and $\eta_e = \sigma\left(  \mathbf{W}^{1}_{\text{gate}} h_e^{(\ell)}  +  \mathbf{W}^{2}_{\text{gate}} \varphi_e + \mathbf{b}_{\text{gate}}  \right)$. Here $\mathbf{W}^{1}_{\text{gate}}, \mathbf{W}^{2}_{\text{gate}}$, and $\mathbf{b}_{\text{gate}}$ are learnable parameters.

\noindent
Equation~\ref{eq:gate} computes entity representations at each neural network layer. A simple approach is to use $\Upsilon_e^{(L)}$, where $L$ is the final layer. However, each layer captures information from the 1-hop, multi-hop, and feature encoders at different granularities. For instance, $\Upsilon_e^{(1)}$ encodes $e$ using its 1-hop, $k$-hop neighborhoods, and features, while $\Upsilon_e^{(2)}$ extends this to 2-hop and $2k$-hop neighborhoods. To leverage the complementary information across layers, we use the following representation:\vspace{-2ex}
\begin{equation}\label{eq:final-enc}
    \mathbf{h}_e = \bigoplus_{\ell = 1}^{L} \frac{\Upsilon_e^{(\ell)}}{||\Upsilon_e^{(\ell)}||_{2}},
\end{equation}
where $\bigoplus$ denotes concatenation and $||.||_2$ is the L2-norm.

\noindent
\textbf{Training and Loss Function.}
We expect aligned entities to have small distances and unaligned entities to have large distances. To enforce this, we use contrastive alignment loss in training.
Let $\mathcal{A}^+$ be the set of ground-truth aligned entities, i.e., $\mathcal{A}^+ = \{ (e_s, e_t) | e_s \in \E_s, e_t \in \E_t \:\: \text{s.t.} \:\: \text{$e_s$ and $e_t$ are the same entity on Earth.}  \}$. 
For unaligned entities, we generate negative samples by replacing one entity in each aligned pair $(e_s, e_t) \in \mathcal{A}^+$ with a randomly selected $e'_s \in \E_s$ or $e'_t \in \E_t$, forming the set $\mathcal{A}^-$. To minimize distances between aligned entities and maximize those between unaligned pairs, we use the following loss function:
\begin{equation}\label{eq:contrast-loss}
    \mathcal{L}_{\text{contrast}} = \hspace{-13pt}\sum_{(e_s, e_t) \in \mathcal{A}^+} \hspace{-13pt} \| \mathbf{h}_{e_s} - \mathbf{h}_{e_t} \|_2 
    + \beta \hspace{-12pt}\sum_{(e'_s, e'_t) \in \mathcal{A}^-} \hspace{-13pt} \max \left\{ 0, \lambda - \| \mathbf{h}_{e'_s} - \mathbf{h}_{e'_t} \|_2 \right\}
\end{equation}
where $\beta$ and $\lambda$ are hyperparameters, and $\|\cdot\|_2$ is L2-norm. Hyperparameter $\lambda$ determines how far the encodings of unaligned entities should be. Also, $\beta$ is used to weight positive and negative samples in the training phase to balance their contribution to the learning process.

\noindent
As discussed, The multi-hop neighborhood of an entity provides complementary information by inferring connection semantics. For instance, in Figure~\ref{fig:2hop}(b), $\mathbf{c} \overset{\text{Bottom''}}{\rightarrow} \mathbf{b}$ and $\mathbf{b} \overset{\text{Top-Right''}}{\rightarrow} \mathbf{a}$ imply that $\mathbf{c}$ is close to and on the right of $\mathbf{a}$. Our model aims to learn entity encodings that capture such inferences and relation semantics. Using the translational assumption~\cite{translating-embeddings}, we interpret $\mathbf{h}_{e} - \mathbf{h}_{e'}$ as the semantic connection between $e$ and $e'$. Thus, if relation $r$ holds between $e$ and $e'$, then $\mathbf{h}_{e}$ should be close to $\mathbf{h}_{e'} + \mathbf{h}_r$. While the loss function defined in Equation~\ref{eq:contrast-loss} can minimize (resp. maximize) the distance between aligned (resp. unaligned) pairs, it cannot capture the semantics of connections in knowledge graphs based on translational assumption. To address this limitation and to avoid the overhead of parameters, for a relation $r\in \R$ we consider its encoding, $\Theta_{r}$, as the average of the differences between its related entity embeddings:
\begin{equation}\label{eq:relation-encoding}
    \Theta_{r} = \frac{1}{|\T_r|} \sum_{(e, r, e') \in \T} \left( \mathbf{h}_e - \mathbf{h}_{e'}\right), 
\end{equation}
where $\T_r = \{ (e, e') | (e, r, e') \in \T \}$. Next, we use the following relation loss for refinement:
\begin{equation}\label{eq:relation-loss}
    \mathcal{L}_{\text{semantics}} = \sum_{r \in \R} \frac{1}{|\T_r|} \sum_{(e, r, e') \in \T} \left|| \mathbf{h}_e - \mathbf{h}_{e'} - \Theta_r \right||.
\end{equation}

\noindent
The semantic loss function is to ensure that the distance between each pair of entities linked by a particular type of connection remains the same (equal to the encoding of the relation type).
To learn both relation semantics and entity encodings, we minimize the weighted sum of contrastive and relational semantic losses:
\begin{equation}\label{eq:final-loss}
    \mathcal{L} = \mathcal{L}_{\text{contrast}} + \alpha \:\mathcal{L}_{\text{semantics}}.
\end{equation}
The training procedure of the model is shown in Algorithm~\ref{alg:KG-encoder-training} in Appendix~\ref{sec:appendix_algorithms}. 

\noindent
\textbf{Map Matching.}
Algorithm~\ref{alg:matching} shows the map-matching process in \model. Given source and target GDBs, $G_s = (E_s, SE_s, F_s)$ and $G_t = (E_t, SE_t, F_t)$, we first construct their corresponding knowledge graph by using the procedure based on Algorithm~\ref{alg:KG-construction}, $\GG_s = (\E_s, \R_s, \T_s, \F_s)$ and $\GG_t = (\E_t, \R_t, \T_t, \F_t)$. Next, we train the knowledge graph encoding model as described in Algorithm~\ref{alg:KG-encoder-training}. Once we train our model by minimizing the objective function in Equation~\ref{eq:final-loss}, we use the entity encodings as well as area similarity to calculate the similarity of each pair of entities $(e, e')$, where $e \in \E_s$ and $e' \in \E_t$.  
As discussed in Section~\ref{ch:KG}, the knowledge graph structure captures the relative positions of entities, making the representation robust to shift and rotation. However, it does not capture the shape and area of objects on the map. Therefore, we define the similarity between two objects $e \in \E_s$ and $e' \in \E_t$ as
\begin{equation}\label{eq:similarity}
    \textsc{Sim}(e, e') =  \tau \times \textsc{Sim}_{\text{KG}}(\mathbf{h}_e, \mathbf{h}_{e'}) +  (1 - \tau) \times \textsc{Sim}_{\text{Area}}(e, e'),
\end{equation}
where $\textsc{Sim}_{\text{KG}}(.)$ captures the similarity of $e$ and $e'$ based on the knowledge graph encoding, and $\textsc{Sim}_{\text{Area}}(.)$ captures the similarity of the position and the shape of $e$ and $e'$ on the map. In our experiments, we use cosine similarity as $\textsc{Sim}_{\text{KG}}(.)$, Jaccard similarity as $\textsc{Sim}_{\text{Area}}(.)$, and $\tau = 0.5$.

\noindent
Next, we aim to find a weighted match between entities, $\mathbf{M} = \{ (e, e')| e \leftrightarrow e' \:\: \text{s.t.}\:\: e \in \E_s, e' \in \E_t\}$, where $e \leftrightarrow e'$ represents the matched entities, such that the cumulative similarity of matched entities be maximized, i.e.:
\begin{equation}\label{eq:matching}
    \mathbf{M} = \underset{\hat{\mathbf{M}} \: \in \: \mathbb{M}}{\arg \max} \sum_{(e, e') \in \hat{\mathbf{M}}}\textsc{Sim}(e, e'),
\end{equation}
where $\mathbb{M}$ is the set of all possible matchings. To solve the above weighted matching problem, we convert it to the problem of finding Minimum Weight Matching in Bipartite Graphs~\cite{matching}, which is solvable in $\mathcal{O}(\max\left\{ |\E_s|, |\E_t| \right\}^3)$. The details of the conversion can be found in Appendix~\ref{sec:maatchalgappendix}.

\subsection{Map Merging}
\label{ch:map-merging}

After matching, some entities remain unmatched between two GDBs since each may contain unique information. To enrich the source GDB, unmatched entities from the target GDB must be merged, but adding them directly can cause inconsistencies. Our map merging approach ensures consistent integration by first adding unmatched entities, then applying operations (e.g., shift, resize) to resolve conflicts. Since shifting one object may create new overlaps, we define operations mathematically and formulate the merging problem as a mixed integer linear program, ensuring no overlaps while minimizing changes to the map.

\noindent
In GDBs, objects, especially non-linear ones (e.g., buildings), can have complex, non-convex shapes, making mathematical formulation challenging. To address this, we approximate all objects, both linear and non-linear, using a Minimum Bounding Rectangle—a rectangle with sides parallel to the $\mathbf{x}$ and $\mathbf{y}$ axes that minimally encloses the shape.
Accordingly, from now on, we approximate each object $e \in E \cup SE$ and represent it by four points $p_1^e = \begin{pmatrix}
    x^e_1 \\
    y_1^e
\end{pmatrix}, p_2^e = \begin{pmatrix}
    x^e_2 \\
    y_1^e
\end{pmatrix}, p_3^e = \begin{pmatrix}
    x^e_1 \\
    y_2^e
\end{pmatrix},$ and $p_4^e = \begin{pmatrix}
    x^e_2 \\
    y_2^e
\end{pmatrix}$, which are the endpoints of its minimum bounding rectangle. We formulate the operations like shifting, reshaping, etc. as the following operation:\vspace{-1ex}
\begin{dfn}[$\mathbf{\epsilon}$-shift Operation]
Given $\varepsilon^x_c, \varepsilon^y_c, \varepsilon^x_1, \varepsilon^y_1, \varepsilon^x_2,$ and $\varepsilon^y_2 \in \mathbb{R}$, let $\mathbf{\epsilon} = (\varepsilon^x_c, \varepsilon^y_c, \varepsilon^x_1, \varepsilon^y_1, \varepsilon^x_2, \varepsilon^y_2)$, we define $e' = (p_1^{e'}, p_2^{e'}, p_3^{e'}, p_4^{e'})$ as the $\mathbf{\epsilon}$-shift of a shape $e = (p_1^e, p_2^e, p_3^e, p_4^e)$, defined as: 
\vspace{-4ex}
\begin{align}\nonumber
    &p_1^{e'} = \begin{pmatrix}
    x^e_1 + \varepsilon^x_1 + \varepsilon^x_c\\
    y_1^e + \varepsilon^y_1 + \varepsilon^y_c
\end{pmatrix}\:\:\:\:\:\:\:\: p_2^{e'} = \begin{pmatrix}
    x^e_2 + \varepsilon^x_2 + \varepsilon^x_c\\
    y_1^e + \varepsilon^y_1 + \varepsilon^y_c
\end{pmatrix}\\
    &p_3^{e'} = \begin{pmatrix}
    x^e_1 + \varepsilon^x_1 + \varepsilon^x_c\\
    y_2^e + \varepsilon^y_2 +\varepsilon^y_c
\end{pmatrix} \:\:\:\:\:\:\:\: p_4^{e'} = \begin{pmatrix}
    x^e_2 + \varepsilon^x_2 + \varepsilon^x_c\\
    y_2^e + \varepsilon^y_2 + \varepsilon^y_c
\end{pmatrix}.
\end{align}
\end{dfn}

\noindent
It is simple to see that $\mathbf{\epsilon}$-shift operation is a generalization of many transitions (e.g., shift, resize, etc.). For example, when $\varepsilon^x_c = \varepsilon_1 = \varepsilon_2 = \varepsilon_3 = \varepsilon_4 = 0$ and $\varepsilon^y_c = \varepsilon$, this operation is equivalent to simple shifting an object by $\varepsilon$ in the direction of y-axis. The main intuition of this operation is that we allow shifting each of the four sides of the rectangle by $\varepsilon_i$ and/or shifting its center by $\varepsilon^x_c$ and $\varepsilon^y_c$.

\noindent
Given two GDBs, \( G_s = (E_s, SE_s, F_s) \) and \( G_t = (S_t, SE_t, F_t) \), we assume the source GDB, \( G_s \), is immutable, following prior studies. Thus, we add unmatched entities from \( G_t \) to \( G_s \) and apply an \( \mathbf{\epsilon} \)-shift operation to resolve inconsistencies. Given overlapping rectangles—blue (target GDB) and black (immutable source GDB)—we determine an optimal \( \mathbf{\epsilon} \)-shift for the blue rectangle to eliminate overlap.  Overlaps fall into three cases: \circled{1} case 1: A blue vertex is inside the black rectangle, \circled{2} case 2: no blue vertex is inside, but a black vertex is within the blue rectangle, and \circled{3} case 3: neither contains a vertex of the other. \autoref{fig:merge} illustrates these cases. We use linear inequalities to ensure no overlap remains after shifting. Since black rectangles are fixed, we formulate constraints accordingly and apply them iteratively for each pair of neighboring rectangles, as identified using an R-tree data structure.Also, we represent the four vertices of the black rectangle by $p_1 = \begin{pmatrix}
    a_1 \\
    b_1
\end{pmatrix}, p_2 = \begin{pmatrix}
    a_2 \\
    b_1
\end{pmatrix}, p_3 = \begin{pmatrix}
    a_1 \\
    b_2
\end{pmatrix}$, and $p_4 = \begin{pmatrix}
    a_2 \\
    b_2
\end{pmatrix}$, and four vertices of the blue vertices by $p'_1 = \begin{pmatrix}
    a'_1 \\
    b'_1
\end{pmatrix}, p'_2 = \begin{pmatrix}
    a'_2 \\
    b'_1
\end{pmatrix}, p'_3 = \begin{pmatrix}
    a'_1 \\
    b'_2
\end{pmatrix}$, and $p'_4 = \begin{pmatrix}
    a'_2 \\
    b'_2
\end{pmatrix}$. Note that after $\epsilon$-shifting the blue rectangle, its vertices are $p'_1 = \begin{pmatrix}
    a'_1 + \varepsilon^x_1 + \varepsilon^x_c \\
    b'_1 + \varepsilon^y_1 + \varepsilon^y_c
\end{pmatrix}, p'_2 = \begin{pmatrix}
    a'_2 + \varepsilon^x_2 + \varepsilon^x_c\\
    b'_1 + + \varepsilon^y_1 + \varepsilon^y_c
\end{pmatrix}, p'_3 = \begin{pmatrix}
    a'_1 + \varepsilon^x_1 + \varepsilon^x_c\\
    b'_2 + + \varepsilon^y_2 + \varepsilon^y_c
\end{pmatrix}$, and $p'_4 = \begin{pmatrix}
    a'_2 + \varepsilon^x_2 + \varepsilon^x_c\\
    b'_2 + + \varepsilon^y_2 + \varepsilon^y_c
\end{pmatrix}$.

\noindent
\textbf{Case 1.} \autoref{fig:merge}(a) shows the first case. To make sure that this case would not happen we have the following statements:
\begin{align*}\label{st:11}
    \hspace*{-1ex}&\textbf{if} \: a_1 \leq a_1' + \varepsilon_1^x + \varepsilon^x_c \leq a_2: \left( b'_1 +  \varepsilon^y_1 + \varepsilon^y_c \geq b_1 \right) \text{or} \left( b_1' + \varepsilon^y_1 + \varepsilon^y_c \leq b_2 \right) \\ 
    \hspace*{-1ex}&\textbf{if} \: a_1 \leq a_1' + \varepsilon_1^x + \varepsilon^x_c \leq a_2 : \left( b'_2 +  \varepsilon^y_2 + \varepsilon^y_c \geq b_1 \right) \text{or} \left( b_2' + \varepsilon^y_2 + \varepsilon^y_c \leq b_2 \right) \\
    \hspace*{-1ex}&\textbf{if} \: a_1 \leq a_2' + \varepsilon_2^x + \varepsilon^x_c \leq a_2: \left( b'_1 +  \varepsilon^y_1 + \varepsilon^y_c \geq b_1 \right) \text{or} \left( b_1' + \varepsilon^y_1 + \varepsilon^y_c \leq b_2 \right) \\
    \hspace*{-1ex}&\textbf{if} \: a_1 \leq a_2' + \varepsilon_2^x + \varepsilon^x_c \leq a_2 : \left( b'_2 +  \varepsilon^y_2 + \varepsilon^y_c \geq b_1 \right)  \text{or} \left( b_2' + \varepsilon^y_2 + \varepsilon^y_c \leq b_2 \right)
\end{align*}

\noindent
Each statement ensures that a specific blue rectangle vertex is not inside the black rectangle. By Lemma~\ref{lemma:first} in Appendix~\ref{sec:merging}, we can translate these statements into linear inequalities within a mixed integer program. Let \( \mathcal{I}_1 \) denote the set of all such inequalities for all candidate rectangle pairs.  Similarly, we derive corresponding statements for cases 2 and 3, which are also transferable to a linear program. Details are provided in Appendix~\ref{sec:cases}. We denote the resulting inequalities for cases 2 and 3 as \( \mathcal{I}_2 \) and \( \mathcal{I}_3 \), respectively.

\noindent
The above statements ensure that an appropriate $\epsilon$-shift can resolve all inconsistencies in $G_t$. Simultaneously, we aim to minimize map changes, including position shifts and scale adjustments during merging. Given all inequalities in $\mathcal{I}_1, \mathcal{I}_2,$ and $\mathcal{I}_3$, we define the map merging problem as the following optimization problem:

\begin{problem}[Map Merging]\label{prob:map-merging-optimization}
    We aim to find variables $\varepsilon^x \in \Upsilon^x, \varepsilon^y \in \Upsilon^y, \varepsilon_c^x \in \Upsilon_c^x,$ and $\varepsilon_c^y \in \Upsilon_c^y$ to $\epsilon$-shift each of the objects such that:
    \begin{align}\nonumber
        & \text{minimize} \quad \gamma \left(\sum_{\varepsilon^x \in \Upsilon^x} |\varepsilon^x| + \sum_{\varepsilon^y \in \Upsilon^y} |\varepsilon^y| \right) +  \sum_{\varepsilon_c^x \in \Upsilon_c^x} |\varepsilon_c^x| + \sum_{\varepsilon_c^y \in \Upsilon_c^y} |\varepsilon_c^y| \\ \nonumber
       & \text{subject to } \quad \{ \mathcal{I}_1, \mathcal{I}_2,  \mathcal{I}_3 \},
\end{align}
where $\Upsilon^x, \Upsilon_c^x$ (resp. $\Upsilon^y, \Upsilon_c^y$) is the set of all variables that we use in $\mathcal{I}_1, \mathcal{I}_2,$ and $\mathcal{I}_3$ to distort rectangles in the $\mathbf{x}$-axis (resp. $\mathbf{y}$-axis). Here, $\Upsilon_c^x$ and $\Upsilon_c^y$ are the set of all variables that we use in $\mathcal{I}_1, \mathcal{I}_2,$ and $\mathcal{I}_3$ to distort the center of rectangles. 
\end{problem}

\noindent
The main intuition of using parameter $\gamma$ in Problem~\ref{prob:map-merging-optimization} is to emphasize keeping the shape of the object and encourage the model to shift the rectangle. In our experiments, we use $\gamma = 2.1$.

 \noindent
By Lemma~\ref{lemma:third} in Appendix~\ref{sec:merging}, since all inequalities in \(\mathcal{I}_i\) (\(i = 1,2,3\)) are in mixed integer programming form, Problem~\ref{prob:map-merging-optimization} can be solved using Mixed Integer Programming Solvers~\cite{MIPS} to find optimal \(\epsilon\)-shifts. The inequalities in \(\mathcal{I}_1, \mathcal{I}_2,\) and \(\mathcal{I}_3\) ensure no intersections remain, preventing new inconsistencies after merging.

\section{Experiments}
\label{ch: EX}

\noindent
\textbf{Datasets.}
We evaluate our approach using sidewalks and buildings data from two real-world geographic datasets: OpenStreetMap (OSM)\cite{osm} and Boston Open Data (BOD)\cite{boston, bostonmain}. OSM is a high-quality, crowd-sourced geographic database continuously updated by volunteers through surveys, aerial imagery, and freely licensed geodata~\cite{osm}. BOD, developed by the City of Boston's GIS Team, provides publicly available data on sidewalks, crosswalks, and buildings~\cite{boston, bostonmain}. Table~\ref{tab:stat} shows dataset statistics.

\begin{table}
\begin{center}
 \caption{Datasets Statistics}
  \vspace{-2ex}
 \resizebox{\columnwidth}{!} {
\begin{tabular}{ l | c | c | c | c | c }
 \toprule
  {Dataset} & \multicolumn{1}{c}{\begin{tabular}{@{}c@{}}Number of\\ sidewalks (ways)\end{tabular}} & 
  \multicolumn{1}{c}{\begin{tabular}{@{}c@{}}Number of\\ terminal nodes\end{tabular}} & 
  \multicolumn{1}{c}{\begin{tabular}{@{}c@{}}Number of\\ intermediate nodes\end{tabular}} & 
  \multicolumn{1}{c}{\begin{tabular}{@{}c@{}}Number of\\ segments\end{tabular}} & 
  \multicolumn{1}{c}{\begin{tabular}{@{}c@{}}Number of\\ buildings\end{tabular}} \\
 \midrule \midrule
    OSM & 9472  & 10280 & 19432 & 13376 & 9228 \\
    \hline
    BOD & 17936  & 14480 & 100916 & 18407 & 11004 \\
 \bottomrule
\end{tabular}}
\label{tab:stat}
\end{center}
\end{table}

\noindent
\textbf{Baselines.}
We compare our approach with state-of-the-art map conflation algorithms. MAYUR \cite{gorisha} models each source edge as a relation and uses the Rank Join algorithm~\cite{schnaitter2008evaluating} to identify matching road segments, merging unmatched segments via rubbersheeting. Hootenanny~\cite{hoot, canavosio2015hootenanny} automates part of the topographic conflation process, but manual Conflict Resolution remains necessary for full database integration~\cite{canavosio2015hootenanny}. For buildings, we compare \model{} with Hootenanny and a baseline using Jaccard similarity for matching. Since we lack access to Hootenanny’s map matching results, we infer missing and incorrect matches based on duplicate and disoriented geometries. Additionally, we evaluate two merging baselines: rubbersheeting, applied to buildings for merging unmatched non-linear objects, and position merging (Position), where entities from the target GDB are added to the source GDB at their exact locations. In both cases, we use our matching process but replace our merging approach with these baselines.

\noindent
\textbf{Experimental Setup.}
During training, we tune hyperparameters using grid search on a validation set of 100 buildings and 100 segments. We search over: \circled{1} $\beta \in {0.1, 0.2, 0.4, 0.8}$, \circled{2} $\alpha \in {0.005, 0.01, 0.05, 0.1, 0.2}$, and \circled{3} $\lambda \in {0.6, 0.8, 1.0, 1.2, 1.4, 1.6, 1.8, 2}$.

\noindent
We use a learning rate of $0.001$, a hidden dimension of 300, and $L=2$ layers in the Graph Neural Network. For contrastive loss, we sample 10 negative pairs per pre-aligned entity. Activation functions are \texttt{GeLU} for the MLP-Mixer, \texttt{ReLU} for the gate module, \texttt{LeakyReLU} for attention~\cite{GAT}, and \texttt{tanh} for neighborhood aggregation. In map matching, we set $\tau = 0.5$, averaging different similarity measures. For merging, we use $\gamma = 2.1$ to prioritize center shift and preserve entity shapes.

\subsection{Map Matching}
To evaluate performance in the map matching task, we measure the number of correct matches (aligned entity pairs that are correct), incorrect matches (aligned pairs that are incorrect), and missing matches (cases where a correct match exists in the target but is not identified).
The concept of correct matches, incorrect matches, and missing matches correspond to true positives, false positives, and false negatives, respectively. Correct matches (true positives) occur when the matching process accurately identifies a match between the source and target databases. Incorrect matches (false positives) arise when a match is incorrectly identified, while missing matches (false negatives) occur when the process fails to detect an actual match.
we use precision and recall to measure the performance of the model. Precision measures the accuracy of positive matches, while recall measures the completeness of positive matches.

\noindent
As discussed in \LL{Appendix}~\ref{ch:RW}, most recent studies focus on linear objects. In this experiment, we manually compare our approach with the state-of-the-art and a baseline, reporting results in \autoref{tab:combined-matching}. \model{} achieves the highest correct matches and the fewest incorrect and missing matches, outperforming \textsc{MAYUR} and Hootenanny in Precision and Recall. Unlike MAYUR and Hootenanny, which rely only on sidewalks for matching, \model{} leverages both buildings and sidewalks, improving accuracy in dense areas with multiple candidates. Its knowledge graph representation captures both relative and absolute positions, enhancing sidewalk matching. While Hootenanny considers buildings when matching them, it does not use this information for sidewalks.

\noindent
Next, we evaluate \model{} in building matching by comparing it with Hootenanny and Jaccard similarity. \autoref{tab:combined-matching} shows that \model{} outperforms both baselines across all metrics. Its correct match rate is significantly higher, as it learns the best possible matching from data, unlike heuristics-based baselines. Jaccard Similarity relies solely on shape overlap, ignoring features and neighborhood structures. \model{} also achieves the lowest incorrect match rate by leveraging features and structural properties. Jaccard Similarity and Hootenanny have higher missing matches due to positional noise in the databases and Hootenanny’s lack of neighborhood~context.

\noindent
Finally, \autoref{tab:combined-matching} reports map matching results for all entities. Hootenanny is the only competitor matching both buildings and sidewalks, with Jaccard similarity as the baseline. \model{} improves precision by $4\%$ and recall by $5.7\%$, achieving the best performance in overall map matching and each subtask—sidewalk and building matching. More discussion on why \model{} outperforms these methods is provided in Appendix~\ref{subsec:fails}.

\begin{table}
\begin{center}
\caption{Comparative performance of models across sidewalk, building, and map matching}
\resizebox{\columnwidth}{!}{
\begin{tabular}{ l | c | c | c | c | c }
\toprule
{Model} & \multicolumn{1}{|p{2cm}|}{\centering Correct \\ Matches} & \multicolumn{1}{|p{2cm}|}{\centering Incorrect \\ Matches} & \multicolumn{1}{|p{2cm}|}{\centering Missing \\ Matches} & {Precision} & {Recall} \\
\midrule\midrule

\multicolumn{6}{c}{\textbf{Sidewalk Matching}} \\
\midrule
\model{} & 0.986 & 0.012 & 0.002 & 0.988 & 0.997 \\
MAYUR & 0.972 & 0.019 & 0.009 & 0.981 & 0.990 \\
Jaccard & 0.794 & 0.086 & 0.120 & 0.902 & 0.868 \\
Hootenanny & 0.891 & 0.041 & 0.068 & 0.956 & 0.929 \\
\hline\hline

\multicolumn{6}{c}{\textbf{Building Matching}} \\
\midrule
\model{} & 0.898 & 0.039 & 0.063 & 0.958 & 0.934 \\
Jaccard & 0.813 & 0.084 & 0.103 & 0.906 & 0.887 \\
Hootenanny & 0.832 & 0.076 & 0.092 & 0.916 & 0.900 \\
\hline\hline

\multicolumn{6}{c}{\textbf{Map Matching}} \\
\midrule
\model{} & 0.942 & 0.025 & 0.033 & 0.974 & 0.966 \\
Jaccard & 0.803 & 0.084 & 0.113 & 0.905 & 0.876 \\
Hootenanny & 0.861 & 0.058 & 0.081 & 0.936 & 0.914 \\

\bottomrule
\end{tabular}
}
\vspace{-2ex}
\label{tab:combined-matching}
\end{center}
\end{table}

\subsection{Experiment 2: Map Merging}
In map merging, we use two metrics: \circled{1} Cumulative normalized inconsistency and \circled{2} Distortion of the merged objects. Cumulative Normalized Inconsistency (CNI) is used to measure the inconsistency in the map. For two non-linear objects $A$ and $B$ we define their normalized inconsistency as 
$\frac{Area(A) \cap Area(B)}{Area(A) \cup Area(B)}.$ 
As discussed in \autoref{ch:map-merging}, our goal is to obtain a merged map without inconsistencies. We use CNI to assess the inconsistency introduced by different methods. \autoref{tab:inconsis} presents results for \model{} and three baselines. The first row shows the CNI in the merged map. The second row reports the new CNI caused by merging, after removing the initial OSM inconsistency. The third row indicates the percentage of new CNI over the initial CNI in the map. \model{} results in the lowest CNI, maintaining the original inconsistency of the OSM dataset without introducing new inconsistencies, as our approach guarantees no additional inconsistency. The position matching method, which places new entities at their target GDB locations, demonstrates the effectiveness of our mixed-integer programming formulation, removing 7.157 CNI from the merged map. Although Rubbersheeting and Hootenanny reduce some inconsistencies, they cannot eliminate them entirely.

\begin{table}
 \caption{Cumulative normalized inconsistency (CNI) in the conflated database and the new CNI caused by merging for \model{} and baselines.}
  \vspace{-1ex}
 \resizebox{0.8\columnwidth}{!} {
\begin{tabular}{ l | c | c| c| c }
 \toprule 
{}
  & \model{}  &{Position} &  {Rubbersheeting}  &  {Hootenanny}\\
 \midrule \midrule
 {CNI}  &  6.464 & 13.621 & 12.047 & 11.106 \\ 
 \midrule
 {New CNI} & 0 & 7.157 & 5.583 & 4.642 \\
  \midrule
  {New CNI $\%$} & $0\%$  & $110.72\%$ & $86.37\%$ & $71.81\%$ \\
 \bottomrule
\end{tabular}
}
\label{tab:inconsis}   
\end{table}

\begin{table}
\begin{center}
 \caption{Table showing the proportion of displaced segments within a set distance from BOD to the conflated database.}
 \vspace{-1ex}
 \resizebox{0.8\columnwidth}{!} {
\begin{tabular}{ l | c | c| c| c}
 \toprule 
{   Model}
  & \multicolumn{1}{|p{2.2cm}|}{\centering Road
segments\\
within 5 m} &  \multicolumn{1}{|p{2.2cm}|}{\centering Road segments not\\ within 5 m} & \multicolumn{1}{|p{2.2cm}|}{\centering  Road segments within 10 m} & \multicolumn{1}{|p{2.2cm}}{\centering  Road segments not within 10 m}  \\
 \midrule \midrule
 \model{}  &  99.01\%  & 0.003 \%  &  99.91\%  &  0.001\%  \\
 \hline 
 MAYUR \: &  97.90\%  &  0.011\% &  99.12 \% &  0.004\% \\
 \bottomrule
\end{tabular}
}
\label{tab:mergesegments}
\end{center}
\end{table}

\noindent
To compare \model{} with the state-of-the-art road merging method MAYUR~\cite{gorisha}, we use their metric: the percentage of segments whose Hausdorff distance to the corresponding merged segment is within $\eta$-meters, measuring perturbation after merging. As shown in \autoref{tab:mergesegments}, \model{} outperforms MAYUR, achieving lower perturbation by minimizing changes for all entities through its mixed-integer programming formulation, while MAYUR does not optimize for this.Overall, \model{} achieves the lowest CNI (inconsistency) while preserving map integrity with minimal modifications.

\begin{table}
\begin{center}
 \caption{Number of segments and buildings in the merged map by \model{} vs. the source and target geospatial databases.}
  \vspace{-1ex} 
  \resizebox{0.7\columnwidth}{!} {

\begin{tabular}{ l | c | c | c}
 \toprule
  {Dataset} & {$OSM$} &  {$BOD$} & {Merged Map} \\
 \midrule \midrule
    Number of segments & 13376  &   18407  &    21618 \\
    \midrule
    Number of buildings &  9228 &      11004   & 13219\\
 \bottomrule
\end{tabular}}
\label{tab:merged-stat}
\end{center}
\end{table}

Finally, in \autoref{tab:merged-stat}, we report the number of segments and buildings in the merged map by \model{}. The table shows that OSM lacks 8,242 segments and 3,991 buildings, which \model{} merges from BOD data to improve coverage.

\section{Conclusions and Future Work}

Geospatial databases (GDBs) are widely used in navigation, ride-sharing, and logistics, but often lack regional coverage or contain missing entities. Map conflation combines two GDBs by adding missing features from the target while keeping the source unchanged. Existing methods face two main challenges: (1) they focus on linear objects like roads and struggle with non-linear features; and (2) they rely on heuristics instead of learning to match entities in a data-driven way. To address this, we propose \model{}, a machine-learning approach with three steps: (1) knowledge graph construction to model spatial relationships; (2) map matching using graph alignment and geospatial encoding; and (3) map merging via mixed-integer linear programming to integrate unmatched features while preserving consistency. Experimental results show that \model{} outperforms the state-of-the-art and baseline methods in map conflation tasks. Future work includes incorporating adaptive grid and buffer sizes based on local map density and object geometry, as well as supporting consistent matching of entities with differing geometric representations across GDBs.

\section*{GenAI Usage Disclosure}
No generative AI was used in the preparation of this work.

\bibliographystyle{ACM-Reference-Format}
\bibliography{Main}
\appendix
\clearpage
\section{Related Work}
\label{ch:RW}

In this section, to situate our research in a broader context, we review related work in four categories: \circled{1} Map matching, \circled{2} Map merging, \circled{3} Graph representation learning for digital maps, and \circled{4} Representation learning on knowledge graphs.

\subsection{Map Matching}\label{sec:rw-map-matching}
Given two geospatial databases, source and target GDB, the problem of map matching aim to match entities in source and target GDB such that match entities correspond to the same physical object on Earth. With the increasing demand for digital maps and the need for improving their coverage and accuracy, the problem of map matching attracts much attention during the past decade~\cite{walter1999matching, li2010optimized, tong2014linear, li2011optimisation, tong2014linear}. The main idea is to assign a similarity score two each pair of entities and match entities with a high similarity score. However, studies are different in how they define similarity metrics and how they find the solution. In this section, we review existing works in two groups: (1) Greedy methods, which greedily match similar objects in the two GDBs. (2) Optimization methods, which formulate the problem as an optimization task and find the optimal solution.  

\subsection{Greedy Heuristic Methods}
The first group of work develops greedy algorithms to find the matches of the spatial objects between two geospatial databases. \emph{Delimited Stroke Oriented} (DSO) algorithm, which is designed by \citet{zhang2008delimited}, aims to match roads (linear objects) in road networks. \emph{DSO} uses Hausdorff distance, angle, orientation, and some other factors to measure the similarity of roads and iteratively matches the most similar roads. This method requires so many pre-defined thresholds (one threshold per similarity factor) to measure the similarity of two roads, which makes it infeasible for real-world scenarios. A similar approach is designed by \citet{zhang2007iterative}, which uses several thresholds to find the potential candidate matches in road networks. \citet{mustiere2008matching} design \emph{NetMatcher}, a map matching framework that matches points and lines between two road networks. \emph{NetMatcher} first finds all the potential candidates for the nodes and arcs in the geospatial databases (road networks), and then filters them based on the connectivity of nodes. This work assumes that one GDB is a subset of the other one, which is an unrealistic setting. \citet{schafers2014simmatching} develop \emph{SimMatching} that builds a similarity matrix for the pairs of potential candidate matches between two given geospatial databases. In every iteration, an object in one database is matched to its candidate match if the topologically connected roads in one database are also connected in the other database.

Similar to the above methods, several greedy procedures are designed for map matching in road networks~\cite{tong2009probability, mustiere2008matching, song2011relaxation, liu2015progressive}. While the overview of these methods is the same, they are different in \circled{1} the definition of the similarity score (e.g., similarity based on shape, angle, orientation and position~\cite{mustiere2008matching, almotairi2018using} or topological connectivity~\cite{mustiere2008matching, tong2009probability}), and \circled{2} the matching level (e.g., point-to-point matching~\cite{song2011relaxation}, line-to-line matching~\cite{liu2015progressive}). More recently, \citet{new-rw-ali} developed an automated conflation pipeline to align state DOT Linear Referencing System (LRS) basemaps with OSM. Their approach leverages Valhalla’s Hidden Markov Model and Viterbi search for probabilistic route matching. While scalable for roadway networks, it is designed specifically for linear LRS-style routes and does not extend to non-linear entities such as buildings.
However, each of these approaches suffers from a subset of the following limitations: \circled{i} Limited to a specific type of object and cannot generalize to all entities in the map. \circled{ii} Assume unrealistic settings that limit their applicability (e.g., one GDB is a subset of the other one). \circled{iii} Assume pre-defined patterns or rules for the matching of entities (similarity of angle or shape) and cannot learn the match entities from the data, limiting their generalizability in different scenarios. \circled{iv} They are heuristic approaches and find suboptimal solutions.

\subsection{Optimization Methods}
Contrary to the first group, which finds a heuristic solution, the second group formulates the problem as an optimization task and finds the optimal solution. \citet{walter1999matching} model the map matching problem as a communication channel, where source GDB (resp. target GDB) is the sender (resp. receiver). The authors suggest using the Buffer Growing method to find potential matches in road networks and then maximize the \emph{mutual information} between the matched pairs of roads. Their calculation of mutual information heavily relies on a manual statistical analysis to compute the conditional probabilities, which limits the applicability of the method.

Some work~\cite{li2010optimized, li2011optimisation} formulated map matching as an assignment problem. To this end, for each linear object in the source GDB, they find all the potential matches in the other GDB using Hausdorff distance and string similarity (modified Hamming distance) between the names of the linear objects. Next, they aim to minimize the dissimilarity between the matched pairs of objects. Similarly, \citet{tong2014linear} design \emph{OILRM}, an optimization and iterative logistic regression method, that employs an optimization component to find the matching pairs such that the total modified Hausdorff distance between the matched pairs is minimized. Next, to improve the accuracy and recall of this method, it iteratively uses logistic regression. All these methods~\cite{li2010optimized, li2011optimisation, tong2014linear} do not take advantage of the topological connectivity of the objects in the road networks and assume that some metainformation is available.

\citet{lei2019optimal} formulated the matching between the roads as a $p$-matching problem in graphs. In this formulation, each road is considered a node in the graph and edges are added from a node in one GDB to a node in the other GDB if the features corresponding to the nodes are less than a certain Hausdorff distance apart, and their Hausdorff distance is considered as the cost of the edge. Now they model the matching problem as a flow in this network such that the total flow in the network does not exceed $p$. This formulation requires solving a flow problem, which is time-consuming, making it infeasible for large-scale geospatial databases. 

Finally, recently, \citet{gorisha} design a framework, \emph{MAYUR}, based on the Rank Join algorithm to match linear objects in two geospatial databases. \emph{MAYUR} treats each edge of the source database as a relation and leverages the Pull Bound Rank Join algorithm ~\cite{schnaitter2008evaluating}, which is a commonly used algorithm in relational databases, to join these edges and find the top-1 join result in order to identify matching pairs. 

Each of these approaches suffers from a subset of the following limitations: \circled{i} Limited to a specific type of object and cannot generalize to all entities in the map. \circled{ii} Assume pre-defined patterns or rules for the matching of entities (similarity of angle or shape) and cannot learn the match entities from the data, limiting their generalizability in different scenarios. \circled{iii} Solving an optimization is a time-consuming task in most cases, which limits their applicability for large-scale GDBs.

\subsection{Map Merging}\label{sec:rw-map-merging}
Map conflation involves transforming features after map matching. The map-matching results are utilized to identify the unmatched entities in one database, denoted as $G_t$, that need to be merged with the spatial entities in the other database, denoted as $G_s$. However, a simple union of the unmatched entities in $G_t$ with the entities in $G_s$ does not guarantee the desired consistency and connectivity in the conflated map.
Therefore, the unmatched entities must undergo appropriate transformations to ensure consistency and proper connectivity with the existing linear entities in the geospatial database. Numerous studies have been conducted to explore the best way to integrate linear entities, such as road networks, and achieve this objective.\\
Rubbersheeting \cite{haunert2005link, gorisha, chen2006automatically, song2008automated, zhang2016automatic, katzil2005spatial} is a widely used and popular technique that has been extensively adopted in various works. It has been considered the state-of-the-art approach in map merging. Rubbersheeting involves distorting the flexible membrane of one map to align with the other map, using the "control points", corresponding points that are matched between the two geospatial databases after map matching,  as the basis for adjustment. The unmatched features undergo deformation based on these distortions between control points and are integrated into the other map.  The displacement between the control points is a shift vector that represents the vector from one point in one geospatial database to its corresponding point in the other 
geospatial database. These approaches vary in terms of the functions they employ for interpolation and obtaining the shift vectors to find the position of unmatched points. \citet{haunert2005link} employs a rubbersheeting technique to merge the databases, assuming that the links between the corresponding spatial features in the databases are provided as input. \citet{zhang2016automatic} introduce a simple map merging method that utilizes displacement vectors obtained from the control points after road matching using the Delimited-Stroke
Oriented (DSO) algorithm as described in ~\cite{zhang2008delimited}. Another type of work uses a set of some predefined constraints that preserve the geometry of the objects, and use them after merging to minimize the Euclidean distance between the coordinates of the points in the source and conflated database~\cite{harrie1999constraint, touya2013conflation}. One of the main drawbacks of these methods is that they cannot guarantee to remove inconsistencies or even not make any inconsistencies. Finally, very recently, \citet{gorisha} utilize a weighted average of shift vectors for matched terminal points, where the weight assigned to each vector is inversely proportional to the distance between the spatial points.

None of these methods can guarantee that they can remove inconsistencies or not make inconsistencies. Accordingly, after merging, several objects might overlap, which makes the merged map undesirable. 

\subsection{Graph Representation Learning for Digital Maps}\label{sec:rw-gml-maps}
With the recent advancement of machine learning on graphs, several studies model the connectivity of roads as graphs~\cite{ETA, gorisha, thesis, ETAsigmodHistorical, ETAGAN, ETAtrajectory, ETAstochasticRecurrent}. Learning the representation of roads on the map results in state-of-the-art performance in many downstream tasks. Travel-time prediction is a problem that has been closely studied in various settings. Earlier works modify convolutional neural networks to be mindful of the spatial properties of the trajectory~\cite{ETAtrajectory} and employ graph neural networks coupled with recurrent mechanisms~\cite{ETAstochasticRecurrent}. Recently, the CurbGAN framework~\cite{ETAGAN} leverages generative adversarial networks for evaluating urban development plans in the context of estimated travel times. 

All existing machine learning models on maps focus on a specific type of object (e.g., road networks), and are unable to encode all the entities in the map. This formulation can cause suboptimal performance since objects in the map are closely related to each other and encoding non-linear objects (resp. linear objects) without considering linear objects (resp. non-linear objects) can cause inconsistency and/or suboptimal performance. In sections~\ref{ch:KG} and \ref{ch:map-matching}, we present an approach to represent digital maps as knowledge graphs and then a graph machine learning model that can learn feature representation of \emph{all} entities in the map.

\subsection{Representation Learning on Knowledge Graphs}\label{sec:rw-kg}
A knowledge graph is a structured representation of knowledge that captures relationships between entities and their attributes. It organizes information in a graph format, where nodes represent entities and edges represent the connections or relationships between these entities, enabling efficient data retrieval and semantic understanding. With the recent advancement of machine learning on graphs, several methods have been developed to learn the structure of knowledge graphs~\cite{zhang2020learning, rossi2021knowledge, kazemi2018simple, steenwinckel2022ink, zhang2019heterogeneous}. However,  most of these methods are designed for learning the node embeddings of a single knowledge graph for a specific task (e.g., entity classification, edge prediction, etc.). All these methods are different from our knowledge graph representation learning method in section~\ref{ch:map-matching}, as we encode the entities of two knowledge graphs of interest in the same embedding space, enabling the model to align (match) entities between the two knowledge graphs. 

Similar to our work, some methods discuss the problem of knowledge graph alignment~\cite{KG-benchmark}, which encodes the entities of two given knowledge graphs in the same embedding space. Knowledge graph alignment aims to match the entities of two knowledge graphs such that matched entities refer to the same real-world identity~\cite{KG-benchmark}. The generic framework of knowledge graph alignment methods is to use an \emph{Embedding Module}, which aims to learn low-dimensional vector representations of entities in both knowledge graphs and then uses an \emph{Alignment Module}, which aims to unify the embeddings of the two KGs into the same vector space so that aligned entities can be identified. Different methods are different in each of these two modules. Some methods use Graph Convolutional Network (GCN)~\cite{wang2018cross} as \emph{Embedding Module}, while others use Graph Attention Networks~\cite{mao2020mraea, liu2020exploring} or Transformers~\cite{transformer-KG}. However, all these methods are unable to capture multi-hop neighborhoods of entities, which play important roles in map matching (see section~\ref{ch:map-matching}). Also, they only use provided entity features as the input of the first layer of the graph neural network, limiting their ability to learn and emphasize features when they are important.

The closest method to our knowledge graph alignment, introduced in section~\ref{ch:map-matching}, is \emph{AliNet} \cite{alinet}, an end-to-end knowledge graph alignment models that uses distant neighbors to expand the overlap between their neighborhood structures and employs an attention mechanism to highlight helpful distant neighbors and reduce noises. While our knowledge graph encoder is inspired by the design of AliNet's architecture, most of its components are different. AliNet employs a vanilla GCN~\cite{GCN} to encode 1-hop neighborhoods, which potentially suffer from over-smoothing and over-fitting. Also, the vanilla GCN cannot learn the importance of the previous hidden state in the updating process of node encodings. Finally, the vanilla GCN treats all messages from different nodes the same, which causes inconsistencies in the scale of messages (see section~\ref{ch:map-matching}). We address the potential of overfitting by using dropuot~\cite{dropout} and use self-feature modeling to learn the importance of the previous hidden state. Finally, we normalize all the weights in each message-passing process, which makes the scale of all messages the same, avoiding inconsistencies. On the other hand, \emph{AliNet} uses features only in the first layer of the message-passing process, missing their importance in deeper layers. However, digital maps often are associated with rich meta-information (e.g., location, tag, object type, etc.) about entities. To this end, we use an MLP-Mixer~\cite{mlp-mixer} module to learn the features of entities and capture both cross-sample and cross-feature dependencies of objects in maps. 

\section{Example of Linear Objects}
\begin{example}
In Figure~\ref{fig:example1}, the nodes $p_2$, $p_3$, $p_4$, and $p_5$ have degrees of 1, 3, 1, and 1, respectively. Therefore, these nodes are considered terminal points. Although $p_1$ has a degree of 2, the angle between $p_1n_1$ and $p_1p_2$ is significant, making it a terminal point as well. In this figure, there are four segments: $(p_1,p_2)$, $(p_1,n_1,n_2,n_3,p_3)$, $(p_3,p_4)$, and $(p_3,p_5)$. The circular error $\delta$ for the terminal point $p_4$ is depicted by a dashed purple circle.
\end{example}

\begin{figure}
    \centering
\includegraphics[width=0.8\linewidth]{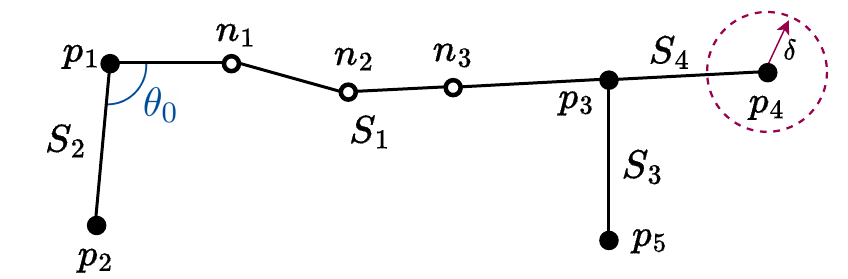}
    \caption{ An example of a linear entity within a geospatial database.}
    \label{fig:example1}
\end{figure}
\label{sec:example}
\section{Knowledge Graph Construction}\label{sec:knconstappendix}
Figure \ref{fig:buff_grid} (a) illustrates a $3\times3$ grid centered at entity $a$ and shows the relative positional relations with respect to $a$.  Figure \ref{fig:buff_grid} (b) shows the buffer around segment \( S_1 = (p_1, p_2) \).

\begin{figure}
\centering
{{\includegraphics[width=0.3\textwidth]{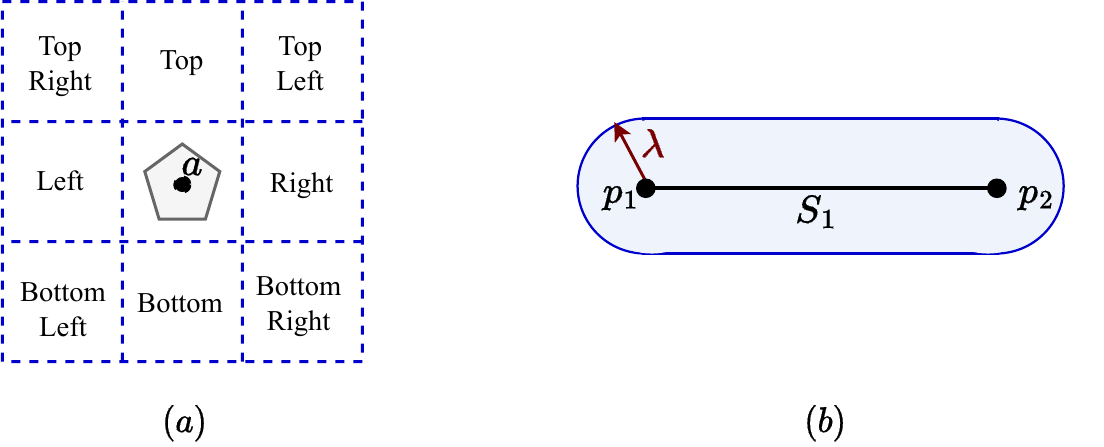} }}

    \caption{(a) An illustration of a 3 × 3 grid centered at entity $a$'s center. Each square in the grid shows the relations with respect to  polygon $a$. We have nine relation types: 1) “Bottom”, 2) “Bottom-Right”, 3) “Right”, 4) “Top-Right”, 5) “Top”, 6) “Top-Left”, 7) “Left”, 8) “Bottom-Left” and 9) “close”. (b) An illustration of the buffer around linear segment $S1 = (p_1, p_2)$, with a width equal to $\lambda$. For segments, we define two relation types, “Inside” if a non-linear entity is inside the buffer of segment $S_1=(p_1, p_2)$, and Connected, if a linear entity has endpoints within the distance $\delta$ of segment $S_1$.}
    \label{fig:buff_grid}
\end{figure}

\section{Map Matching}
\label{sec:matchingappend}
\begin{figure}
\centering
{{\includegraphics[width=\columnwidth]{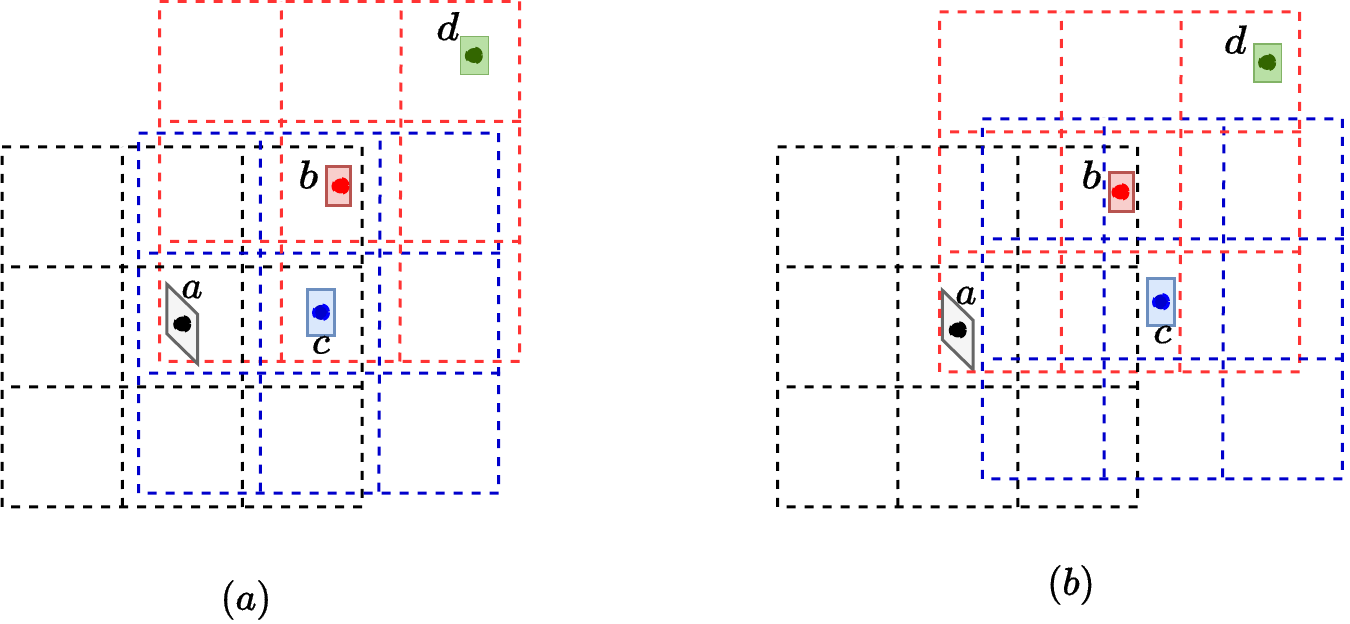} }}
    \caption{Higher-order neighborhoods provide information to mitigate the noise in the location of entities. (a) Source GDB (b) Target GDB.}
    \label{fig:2hop}
\end{figure}

\begin{example}\label{ex:2hop}
    In Figure~\ref{fig:2hop}(a) (resp. Figure~\ref{fig:2hop}(b)), we let $\GG_s = (\E_s, \R_s, \T_s, \F_s)$ (resp. $\GG_t = (\E_t, \R_t, \T_t, \F_t)$) denote the corresponding knowledge graph of the source GDB (resp. target GDB), obtained from Algorithm~\ref{alg:KG-construction}. In $\GG_s$, $\mathbf{c}$ is connected to $\mathbf{a}$ via a connection with type ``Right'' as in the source GDB, $\mathbf{c}$ is inside of the grid around $\mathbf{a}$. However, in $\GG_t$, due to a small noise in the location of $\mathbf{c}$ in the target GDB, $\mathbf{c}$ is outside of the $\mathbf{a}$'s grid. Consequently, in $\GG_t$, $\mathbf{a}$ and $\mathbf{c}$ are not connected. Although this structural difference in the constructed knowledge graphs might seem a challenge in the matching process, considering the higher-order neighborhood of each entity can mitigate this challenge. That is, although in $\GG_t$ we cannot directly conclude that $\mathbf{a}$ and $\mathbf{c}$ are neighbors and $c$ is on the right side of $\mathbf{a}$, they still are in the 2-hop neighborhood of each other and so are close. Moreover, based on the relations of  $\GG_t$, we can conclude that $\mathbf{c}$ is on the right side of $\mathbf{a}$, as in the knowledge graph we have $\mathbf{c} \overset{\text{``Bottom''}}{\rightarrow} \mathbf{b}$ and $\mathbf{b} \overset{\text{``Top-Right''}}{\rightarrow} \mathbf{a}\:$. Accordingly, looking at the 2-hop neighborhood of $\mathbf{a}$ and $\mathbf{c}$ in $\GG_t$ can provide us with complementary information that helps in the matching process of knowledge graphs $\GG_s$ and $\GG_t$.
\end{example}

\begin{example}
    In Figure~\ref{fig:2hop}(b), while both $\mathbf{c}$ and $\mathbf{d}$ are in the 2-hop neighborhood of $\mathbf{a}$, the contribution of $\mathbf{c}$ in $\mathbf{a}$'s neighborhood encoding is more important than $\mathbf{d}$'s, as $\mathbf{d}$ is far away from $\mathbf{a}$. Note that, while the actual distances are not encoded by the knowledge graph directly, relative distances can be inferred from the relation types. 
\end{example}

\begin{example}
    In Figure~\ref{fig:2hop}(a) (resp. Figure~\ref{fig:2hop}(b)), we let $\GG_s = (\E_s, \R_s, \T_s, \F_s)$ (resp. $\GG_t = (\E_t, \R_t, \T_t, \F_t)$) to denote the corresponding knowledge graph of the source GDB (resp. target GDB), obtained from Algorithm~\ref{alg:KG-construction}. We also let $\mathbf{h}^s_e$  (resp. $\mathbf{h}^t_{e'}$) represent the encoding of entity $e \in \GG_s$  (resp. $e' \in \GG_t$) obtained from Equation~\ref{eq:final-enc}. As discussed in Example~\ref{ex:2hop}, $\mathbf{c}$ is not in the 1-hop neighborhood of $\mathbf{a}$ in $\GG_t$ while in $\GG_s$, we have $\mathbf{c} \overset{\text{``Right''}}{\rightarrow} \mathbf{a}\:$. However, from the connection of $\mathbf{b}$ to $\mathbf{a}$ and $\mathbf{c}$ in $\GG_t$, we can conclude that $\mathbf{c}$ is on the right side of and close to $\mathbf{a}$. Based on the translational assumption, we expect that $\mathbf{h}^t_{\mathbf{c}} - \mathbf{h}^t_{\mathbf{a}}$ represents the relation type ``Right''. Accordingly, we expect the model to encode entities such that $|| \mathbf{h}^t_{\mathbf{c}} - \mathbf{h}^t_{\mathbf{a}} || \approx || \mathbf{h}^s_{\mathbf{c}} - \mathbf{h}^s_{\mathbf{a}} || $, which also helps the matching process.
\end{example}
\subsection{Matching Process}
\label{sec:maatchalgappendix}
 To solve the above weighted matching problem we convert it to the problem of finding Minimum Weight Matching in Bipartite Graphs~\cite{matching}. To this end, we construct a complete weighted bipartite graph $\mathcal{B} = (V_1, V_2, V_1 \times V_2, W)$ such that $V_1 = \E_s$, $V_2 = \E_t$ and $W[e, e'] = 1 - \textsc{Sim}(e, e')$. Now, we use the algorithm presented by \citet{matching} and find the minimum weighted matching in the bipartite graph $\mathcal{B}$, $\Tilde{\mathbf{M}}$. We next show that $\Tilde{\mathbf{M}}$ is the best matching for Equation~\ref{eq:matching}.
\begin{prop}\label{prop:matching}
    Given the minimum weighted matching in the bipartite graph $\mathcal{B}$, $\Tilde{\mathbf{M}}$, we have:
    \begin{equation}
        \Tilde{\mathbf{M}} = \underset{\hat{\mathbf{M}} \in \mathbb{M}}{\arg \max} \sum_{(e, e') \in \hat{\mathbf{M}}}\textsc{Sim}(e, e').
    \end{equation}
\end{prop}
\begin{proof}
    We show the above proposition by contradiction. Let $\mathbf{M} \neq \Tilde{\mathbf{M}}$ be the weighted matching with maximum cumulative similarity, i.e.: 
    \begin{equation*}
        \mathbf{M} = \underset{\hat{\mathbf{M}} \in \mathbb{M}}{\arg \max} \sum_{(e, e') \in \hat{\mathbf{M}}}\textsc{Sim}(e, e').
    \end{equation*} 
    Therefore, it is also a matching for bipartite graph $\mathcal{B}$. Since $\Tilde{\mathbf{M}}$ has the minimum weights among all possible matching in the bipartite graph $\mathcal{B}$, we have:
    \begin{align}\nonumber
        &\sum_{(e, e') \in \Tilde{\mathbf{M}}} \left(1 - \textsc{Sim}(e, e')\right) \leq \sum_{(e, e') \in \mathbf{M}} \left(1 - \textsc{Sim}(e, e')\right)\\ \nonumber
        \Rightarrow &\min\{|\E_s|, |\E_t|\} -\!\!\!\!\!\! \sum_{(e, e') \in \Tilde{\mathbf{M}}} \!\!\!\textsc{Sim}(e, e') \leq \min\{|\E_s|, |\E_t|\} - \!\!\! \sum_{(e, e') \in \mathbf{M}}\!\! \!\!\!\textsc{Sim}(e, e')\\
         \Rightarrow &\sum_{(e, e') \in \mathbf{M}} \textsc{Sim}(e, e') \leq \sum_{(e, e') \in \Tilde{\mathbf{M}}} \textsc{Sim}(e, e'),
    \end{align}
    which is a contradiction, since $\mathbf{M}$ has the maximum cumulative similarity among all possible matchings. 
\end{proof}
Based on Proposition~\ref{prop:matching}, the matching found by \citet{matching}, which models the matching problem as a max flow problem in the graphs, in a bipartite graph $\mathcal{B}$, is the solution to our maximum weighted matching problem. Constructing the weighted bipartite graph requires $\mathcal{O}(|\E_s| \times |\E_t|)$ time, and finding the minimum weighted matching in the bipartite graph requires $\mathcal{O}(\max\left\{ |\E_s|, |\E_t| \right\}^3)$. Accordingly, the time complexity of map matching is $\mathcal{O}(\max\left\{ |\E_s|, |\E_t| \right\}^3)$.

\section{Algorithms}
\label{sec:appendix_algorithms}
Algorithm~\ref{alg:KG-construction} describes this process and constructs a knowledge graph for a geospatial database given sets of non-linear entities (polygons) $E$, a set of linear segments $SE$, circular error $\delta$, buffer width $\lambda$ for linear segments, and grid width $\mu$ for non-linear entities. While the above procedure is simple, its naive implementation can cause $\mathcal{O}(|E \cup SE |^2)$ operations, which is time-consuming for large GDBs. That is, for each entity $e \in E \cup SE $, checking all other entities to see whether they are neighbors with $e$ requires $\mathcal{O}(|E \cup SE |^2)$ operations. To this end, we design an efficient algorithm, which is illustrated in Algorithm~\ref{alg:KG-construction}. Algorithm~\ref{alg:KG-construction} first uses R-Tree data structure~\cite{R-tree} to index entities in the GDB, as indicated in line 1 and line 2. This indexing can reduce the average search time to $\mathcal{O}(\log |E \cup SE |)$.
Next, for each entity $e$, a grid with width $\mu$ centered at $e$ is created. Using the constructed indices, we retrieve all entities within the grid,  as indicated in line 4 and line 5 of Algorithm ~\ref{alg:KG-construction}. Note that given the spatial coordinates of entity $e$ as $(e_x, e_y)$, the spatial coordinates  of its corresponding grid can be represented by the bounding box $(x_{min}, y_{min}, x_{max}, y_{max}) = (e_x - \mu/2, e_y - \mu/2, e_x + \mu/2, e_y + \mu/2)$. For each entity $u$ within the grid, its placement  within the grid is determined by its relative position to the location of $e$, as shown in lines 6-27 of Algorithm ~\ref{alg:KG-construction}.  Based on this,  we construct the triple $(v_e, r, u)$, where $r$ denotes their relative relation.

Next, we construct the entities corresponding to linear segments.  Firstly, for each segment $s$, a buffer of width $\lambda$ is created around the segment $s$, as shown in line 29 of Algorithm ~\ref{alg:KG-construction}. Subsequently, utilizing the constructed R-Tree indices, all entities (resp. segments) that fall within (resp. intersect with) the buffer of segment $s$ are retrieved and stored in $\mathcal{L}_s$ (resp. $\mathcal{CS}_s$), as indicated in lines 30 and 31 of the algorithm ~\ref{alg:KG-construction}. For every entity $u$ in $\mathcal{L}_s$, the triple $(u, "Inside", v_s)$ is constructed and added to the knowledge graph, as demonstrated in lines 32 and 33. Additionally, for each segment $s_0$ in $\mathcal{CS}_s$, if it has an endpoint within a distance $\delta$ of segment $s$, the triple $(v_s, "Connected", s_0)$ is formed and added to the knowledge graph, as shown in lines 34-36 of Algorithm  ~\ref{alg:KG-construction}.

\begin{algorithm}[t]
    \small
    \caption{Knowledge Graph Construction}
    \label{alg:KG-construction}
    \begin{algorithmic}[1]
        \Require{$E, SE, \delta, \lambda, \mu$}
        \Ensure{Knowledge Graph $\GG$}
        \State $\mathcal{T} \gets \textsl{BuildRTree}(E)$
        \State $\mathcal{S} \gets \textsl{BuildRTree}(SE)$
         \For{$e \in E$}
                 \State J = Local grid with width $\mu$ centered at $e$
         \State $\mathcal{N}_e \gets \textsl{GetNodesInsideGrid(J)} $
        \For{$u \in \mathcal{N}_e$}
        \If{$e_x+\frac{\mu}{6} <u_x <e_x+\frac{\mu}{2}$}
            \If{$e_y+\frac{\mu}{6} <u_y <e_y+\frac{\mu}{2}$}
            \State connect $v_e$ to $u$ with a relation type of "Top-Right"
            \ElsIf{$e_y-\frac{\mu}{6} <u_y <e_y+\frac{\mu}{6}$}
            \State connect $v_e$ to $u$ with a relation type of "Right"
            \ElsIf{$e_y-\frac{\mu}{2} <u_y <e_y-\frac{\mu}{6}$}
            \State connect $v_e$ to $u$ with a relation type of "Bottom-Right"
            \EndIf
          \ElsIf{$e_x-\frac{\mu}{6} <u_x <e_x+\frac{\mu}{6}$}
            \If{$e_y+\frac{\mu}{6} <u_y <e_y+\frac{\mu}{2}$}
            \State connect $v_e$ to $u$ with a relation type of "Top"
            \ElsIf{$e_y-\frac{\mu}{6} <u_y <e_y+\frac{\mu}{6}$}
            \State connect $v_e$ to $u$ with a relation type of "Close"
            \ElsIf{$e_y-\frac{\mu}{2} <u_y <e_y-\frac{\mu}{6}$}
            \State connect $v_e$ to $u$ with a relation type of "Bottom"
            \EndIf
            
            \ElsIf{$e_x-\frac{\mu}{2} <u_x <e_x-\frac{\mu}{6}$}
            \If{$e_y+\frac{\mu}{6} <u_y <e_y+\frac{\mu}{2}$}
            \State connect $v_e$ to $u$ with a relation type of "Top-Left"
            \ElsIf{$e_y-\frac{\mu}{6} <u_y <e_y+\frac{\mu}{6}$}
            \State connect $v_e$ to $u$ with a relation type of "Left"
            \ElsIf{$e_y-\frac{\mu}{2} <u_y <e_y-\frac{\mu}{6}$}
            \State connect $v_e$ to $u$ with a relation type of "Bottom-Left"
            \EndIf
            
        \EndIf

        \EndFor
             \EndFor
            \For{$s \in SE$}
            \State $B \gets \textsl{createBuffer}(s, \lambda)$
         \State $\mathcal{L}_s \gets \textsl{GetNodesInsideBuffer(B, E)} $
            \State $\mathcal{CS}_s \gets \textsl{GetSegmentsIntersectingBuffer(s, $\mathcal{S}$)} $
        \For{$u \in \mathcal{L}_s$}
        \State connect $v_s$ to $u$ with a relation type of "Inside"
        \EndFor
        \For{$s_0 \in \mathcal{CS}_s$}

                \If{$s_0$ has endpoints within $\delta$ of $s$}
                    \State connect $v_s$ to $v_{s_0}$ with a relation type of "Connected"
                
                \EndIf

        \EndFor
        
             \EndFor     
         
        \State \Return $\GG$
    \end{algorithmic}
\end{algorithm}

\newpage

\begin{algorithm}
    \small
    \caption{Knowledge Graph Encoder Training}
    \label{alg:KG-encoder-training}
    \begin{algorithmic}[1]
        \Require{Two knowledge graphs $\GG_s = (\E_s, \R, \T_s, \F_s)$ and $\GG_t = (\E_t, \R, \T_t, \F_t)$, hyperparameters $k, \lambda, \beta, \alpha$, and set of aligned entities $\mathcal{A}^+$.}
        \Ensure{Trained Knowledge Graph Encoder Model $\mathcal{M}$}
        \State Randomly initialize all learnable parameters; 
        \State Initialize $h_{s_e}^{(0)}, \psi_{s_e}^{(0)}$, $h_{t_e}^{(0)}$ and $\psi_{t_e}^{(0)}$ with feature vectors $\F(e)$;
        \State Let $\mathbf{F}_s$ (resp. $\mathbf{F}_t$) be a matrix with rows $\F(e)$ for $e \in \E_s$ (resp. $e \in \E_t$);        
         \State $\mathbf{H}^s_{\text{token}} \leftarrow \mathbf{F}_s +  \mathbf{W}_{\text{token}}^{(2)} \sigma\left(   \mathbf{W}_{\text{token}}^{(1)} \texttt{LayerNorm}\left(  \mathbf{F}_s\right)  \right)$;
        \State $\Phi^s_{\text{out}} = [\bigoplus_{e \in \E_s } \varphi^s_e]^T \leftarrow \texttt{MLP} \left(\mathbf{H}^s_{\text{token}} +  \sigma\left(   \texttt{LayerNorm}\left(  \mathbf{H}^s_{\text{token}}\right) \mathbf{W}_{\text{channel}}^{(1)} \right) \mathbf{W}_{\text{channel}}^{(2)}\right)$;
        
        \State $\mathbf{H}^t_{\text{token}} \leftarrow \mathbf{F}_t +  \mathbf{W}_{\text{token}}^{(2)} \sigma\left(   \mathbf{W}_{\text{token}}^{(1)} \texttt{LayerNorm}\left(  \mathbf{F}_t\right)  \right)$;
        \State $\Phi^t_{\text{out}} = [\bigoplus_{e \in \E_t } \varphi^t_e]^T \leftarrow \texttt{MLP} \left(\mathbf{H}^t_{\text{token}} +  \sigma\left(   \texttt{LayerNorm}\left(  \mathbf{H}^t_{\text{token}}\right) \mathbf{W}_{\text{channel}}^{(1)} \right) \mathbf{W}_{\text{channel}}^{(2)}\right)$;
        \For{$\ell = 1, \dots, L$}
            \For{$e \in \E_s$}
                \State $h^{(\ell + 1)}_{s_e} = \textsc{Gnn}\left( h_{s_e}^{(\ell)} \right)$;
                \State $\alpha^s_e = [ \bigoplus_{e' \in N^s_k(e)} \alpha^s_{e, e'}] \leftarrow \textsc{Attn}\left( \psi_{s_e}^{(\ell)}, \bigoplus_{e' \in N^s_k(e)} \psi_{s_{e'}}^{(\ell)} \right)$ 
                \State $\psi^{(\ell + 1)}_{s_e} \leftarrow
                \textsc{Gat}\left( \psi_{s_e}^{(\ell)}, \mathbf{\alpha}^s_e \right)$; 
                \State $\zeta^{s}_e \leftarrow \sigma\left(  \mathbf{W}^{1}_{\text{gate}} \psi_{s_e}^{(\ell)}  +  \mathbf{W}^{2}_{\text{gate}} \varphi_{e}^{s} + \mathbf{b}_{\text{gate}}  \right)$;
                \State $\eta^{s}_e \leftarrow \sigma\left(  \mathbf{W}^{1}_{\text{gate}} h_{s_e}^{(\ell)}  +  \mathbf{W}^{2}_{\text{gate}} \varphi_{e}^{s} + \mathbf{b}_{\text{gate}}  \right)$;
                \State $\Upsilon^{(\ell)}_{s_e} \leftarrow \textsc{Gate}\left(h^{(\ell + 1)}_{s_e}, \psi^{(\ell + 1)}_{s_e}, \varphi^s_e, \zeta^{s}_e,  \eta^{s}_e\right)$;
                \Comment{Equation~\ref{eq:gate}}
            \EndFor
            \For{$e \in \E_t$}
                \State $h^{(\ell + 1)}_{t_e} = \textsc{Gnn}\left( h_{t_e}^{(\ell)} \right)$;
                \State $\alpha^t_e = [ \bigoplus_{e' \in N^t_k(e)} \alpha^t_{e, e'}] \leftarrow \textsc{Attn}\left( \psi_{t_e}^{(\ell)}, \bigoplus_{e' \in N^t_k(e)} \psi_{t_{e'}}^{(\ell)} \right)$ ;
                \State $\psi^{(\ell + 1)}_{t_e} \leftarrow
                \textsc{Gat}\left( \psi_{t_e}^{(\ell)}, \mathbf{\alpha}^t_e \right)$; 
                \State $\zeta^{t}_e \leftarrow \sigma\left(  \mathbf{W}^{1}_{\text{gate}} \psi_{t_e}^{(\ell)}  +  \mathbf{W}^{2}_{\text{gate}} \varphi_{e}^{t} + \mathbf{b}_{\text{gate}}  \right)$;
                \State $\eta^{t}_e \leftarrow \sigma\left(  \mathbf{W}^{1}_{\text{gate}} h_{t_e}^{(\ell)}  +  \mathbf{W}^{2}_{\text{gate}} \varphi_{e}^{t} + \mathbf{b}_{\text{gate}}  \right)$;
                \State $\Upsilon^{(\ell)}_{t_e} \leftarrow \textsc{Gate}\left(h^{(\ell + 1)}_{t_e}, \psi^{(\ell + 1)}_{t_e}, \varphi^t_e, \zeta^{t}_e,  \eta^{t}_e\right)$;
                \Comment{Equation~\ref{eq:gate}}
            \EndFor
        \EndFor
         \State $\mathbf{h}^s_e \leftarrow \bigoplus_{\ell = 1}^{L} ||\Upsilon_{s_e}^{(\ell)}||_{2}$ and $\mathbf{h}^t_e \leftarrow \bigoplus_{\ell = 1}^{L} ||\Upsilon_{t_e}^{(\ell)}||_{2}$; \Comment{Equation~\ref{eq:final-enc}}
         \For{$r \in \R$}
            \State $\Theta^s_{r} \leftarrow \frac{1}{|\T^s_r|} \sum_{(e, r, e') \in \T_s} \left( \mathbf{h}^s_e - \mathbf{h}^s_{e'}\right)$; \Comment{Equation~\ref{eq:relation-encoding}}
            \State $\Theta^t_{r} \leftarrow \frac{1}{|\T^t_r|} \sum_{(e, r, e') \in \T_t} \left( \mathbf{h}^t_e - \mathbf{h}^t_{e'}\right)$;\Comment{Equation~\ref{eq:relation-encoding}}
         \EndFor
         \State $\mathcal{A}^- \leftarrow$ GenerateNegativeSamples$\left( \mathcal{A}^+ \right)$;
         \State Calculate the contrastive and semantics losses; \Comment{Equations~\ref{eq:contrast-loss}, \ref{eq:relation-loss}}
         \State Minimize $\mathcal{L}_{\text{contrast}} + \alpha \mathcal{L}_{\text{semantics}}$ using Adam optimizer; \Comment{Equation~\ref{eq:final-loss}}
         \State \Return Trained model $\mathcal{M}$;
    \end{algorithmic}

\end{algorithm}

\begin{algorithm}[t]
    \small
    \caption{Map Matching}
    \label{alg:matching}
    \begin{algorithmic}[1]
        \Require{Two geospatial databases $G_s=(E_s, SE_s, F_s)$ and $G_t = (E_t, SE_t, F_t)$, $k, \lambda, \beta, \alpha, \tau$, and set of aligned entities $\mathcal{A}^+$.}
        \Ensure{A matching between entities in $G_s$ and $G_t$}
        \State $\GG_s = (\E_s, \R, \T_s, \F_s) \leftarrow \text{ConstructKG}(G_s)$; \Comment{Using Algorithm~\ref{alg:KG-construction}}
        \State $\GG_t = (\E_t, \R, \T_t, \F_t) \leftarrow \text{ConstructKG}(G_t)$; \Comment{Using Algorithm~\ref{alg:KG-construction}}
        \State $\mathcal{M} \leftarrow \text{KG-EncoderTraining}(\GG_s, \GG_t, k, \lambda, \beta, \alpha, \mathcal{A}^+)$;\Comment{Using Algorithm~\ref{alg:KG-encoder-training}}
        \State $\mathbf{S} \leftarrow [0]_{|E_s| \times |E_t|}$;
        \State $\mathbf{S}_{\text{linear}} \leftarrow [0]_{|SE_s| \times |SE_t|}$;
        \For{$e \in E_s$}
            \For{$e' \in E_t$}
            \State $\mathbf{S}[e, e'] \leftarrow \tau \times \textsc{Sim}_{\text{KG}}(\mathcal{M}(e), \mathcal{M}(e')) +  (1 - \tau) \times \textsc{Sim}_{\text{Area}}(e, e')$ ;
            \EndFor
        \EndFor
        \For{$e \in SE_s$}
            \For{$e' \in SE_t$}
            \State $\mathbf{S}_{\text{linear}}[e, e'] \leftarrow \tau \times \textsc{Sim}_{\text{KG}}(\mathcal{M}(e), \mathcal{M}(e')) +  (1 - \tau) \times \textsc{Sim}_{\text{Area}}(e, e')$ ;
            \EndFor
        \EndFor
        \State Construct complete bipartite $\mathcal{B} = (E_s, E_t, E_s \times E_t)$ with weights $1 - \mathbf{S}$;
        \State Construct complete bipartite $\mathcal{B}_{\text{linear}} = (SE_s, SE_t, SE_s \times SE_t)$ with weights $1 - \mathbf{S}_\text{linear}$;
        \State $\mathbf{M} \leftarrow $ find the minimum weighted matching in bipartite graphs $\mathcal{B}$, $\mathcal{B}_{\text{linear}}$; \Comment{Using~\cite{matching}}
         \State \Return $\mathbf{M}$;
    \end{algorithmic}
\end{algorithm}

\section{Map Merging}
\label{sec:merging}

\begin{figure*}[h]
    \centering
\includegraphics[width=0.8\textwidth]{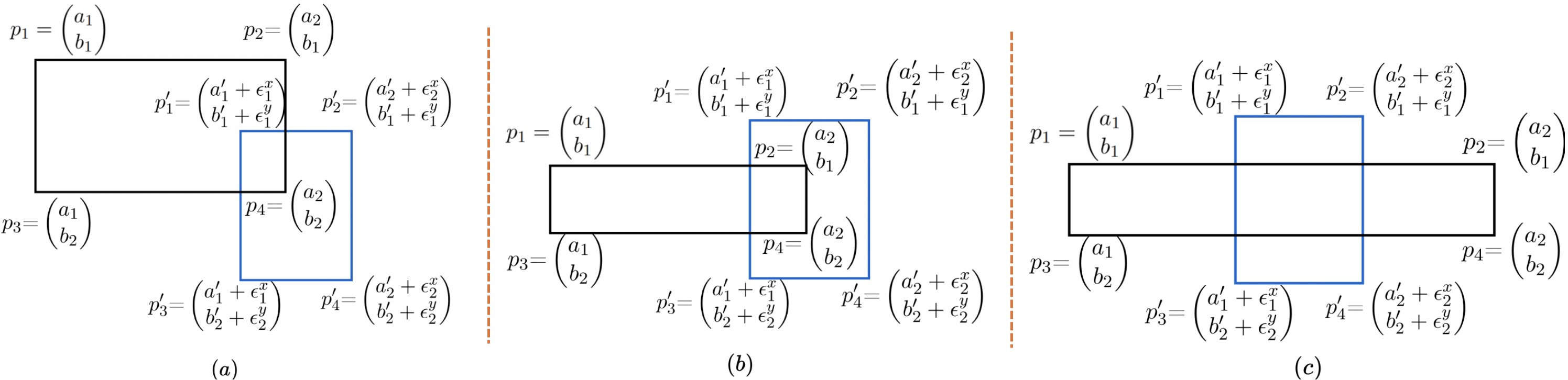}
    \caption{ Illustration of scenarios where an unmatched blue rectangle from the target database intersects with a black rectangle in the source database: (a) A blue rectangle vertex is inside the black rectangle, (b) a black rectangle vertex is inside the blue rectangle, (c) they share an intersecting area without containing each other's vertices.}
    \label{fig:merge}
\end{figure*}

\begin{lemma}
 \label{lemma:generalized}
    Every conjunctive clause of the form 
    \begin{equation}
    \label{eq:general}
    	\bigwedge \limits_{i=1}^n (x_i \theta_i a_i)
    \end{equation}
    where $\theta_i \in \{<,\leq,>,\geq\}$, can be transformed into finite linear inequalities.
\end{lemma}

\begin{proof}
    We define $n$ new binary variables $u_i$ corresponding to each condition. Next, we prove for the case that $\theta_i \in \{\leq, \geq \}$. The proof for the case that we have strict inequality is the same. For $x_i \geq a_i $ (resp. $x_i \leq a_i $ ) we add two inequalities $x_i + M(1-u_i) \geq a_i$ and $x_i-Mu_i<a_i$ (resp.  $x_i \leq a_i + M(1-u_i) $ and $x_i>a_i-Mu_i$), where $M$ is a large number. The intuition is that the binary variable $u_i$ equals one if and only if condition $i$ holds. At the end, we define a new binary variable $t$, and add an inequality $0\leq \sum_{i=1}^{n} u_i -nt \leq1$. 
In fact, $t$ equals one if and only if $u_i = 1 \:\:\forall i$, or in other words, when all conditions hold.
\end{proof}

Note that, if replace the Equation~\ref{eq:general} with a statement with disjunctive clauses, then we can use De Morgan's law and convert it to conductive clauses. Next, we discuss the special case (restricted version) of this lemma, when $n~=~2$.

\begin{lemma}
 \label{lemma:first}
    Every proposition in the form of 
    \begin{equation}
    \label{eqn:first_ineq}
        \text{if}\:\: a_1 \leq x_1 \text{ and } x_2 \leq a_2 \:\: \overset{\text{then}}{\Rightarrow}  \:\: y_1 \geq b_1 \:\: \text{or} \:\: y_2 \leq b_2,
    \end{equation}
    can be transformed into finite linear inequalities.
\end{lemma}

\begin{proof}
We can transform equation~\ref{eqn:first_ineq} into following linear inequations:
\begin{align*}
&\circled{1}\:\: x_1+M(1-u) \geq a_1 &\\
&\circled{2}\:\: x_1- Mu < a_1 &\\
&\circled{3}\:\: x_2 \leq a_2+M(1-v)&\\
&\circled{4}\:\: x_2 > a_2-Mv &\\
&\circled{5}\:\: 0\leq u+v-2t\leq 1 &\\
&\circled{6}\:\: y_1 \geq b_1w +M(w-1)+M(t-1) &\\
&\circled{7}\:\: y_2 \leq b_2 +Mw-M(t-1)&\\
& & \hfill
\end{align*}
where $u,v,w,t \in \{0,1\}$ and $M$ is a large number.\\
In the above inequations, $u$ is $1$ if and only if $a_1 \leq x_1$. Similarly, $v$  is $1$ if and only if $x_2 \leq a_2$.
Based on inequations $\circled{5}$, $t$ is equal to $1$ only when the condition in the implication~\ref{eqn:first_ineq} is satisfied. In fact, $t$ equals $1$ if and only if both $u$ and $v$ are equal to $1$. Furthermore, based on  $\circled{6}$ and $\circled{7}$, when $t = 1$, either $y_1 \geq b_1$ (when $w=1$) or $y_2 \leq b_2$ (when $w=0$). Conversely, when $t = 0$, those two inequalities become trivial.

\end{proof}

\begin{lemma}
    \label{lemma:second}
    Every proposition in the form of 
    \begin{equation}
    \label{eqn:second_ineq}
        \text{if}\:\: a_1 \leq x_1 \text{ and } x_2 \leq a_2 \:\: \overset{\text{then}}{\Rightarrow}  \:\: (y_1-b_1)(y_2-b_2)\geq 0,
    \end{equation}
    can be transformed into finite linear inequalities.
    
\end{lemma}
\begin{proof}
We can transform equation~\ref{eqn:second_ineq} into following linear inequations:
\begin{align*}
&\circled{1}\:\: x_1+M(1-u) \geq a_1 &\\
&\circled{2}\:\: x_1- Mu < a_1 &\\
&\circled{3}\:\: x_2 \leq a_2+M(1-v)&\\
&\circled{4}\:\: x_2 > a_2-Mv &\\
&\circled{5}\:\: 0\leq u+v-2t\leq 1 &\\
&\circled{6}\:\: y_1-b_1\geq-Mw-M(1-t) &\\
&\circled{7}\:\: y_1-b_1\leq M(1-w)+M(1-t) &\\
&\circled{8}\:\: y_2-b_2 \geq-Mw-M(1-t) &\\
&\circled{9}\:\: y_2-b_2\leq M(1-w)+M(1-t) &\\
& & \hfill
\end{align*}
where $u,v,w,t \in \{0,1\}$ and $M$ is a large number.\\
In the above inequations, $u$ is one if and only if $a_1 \leq x_1$. Similarly, $v$  is one if and only if $x_2 \leq a_2$.
Based on $\circled{5}$, $t$ is equal to one only when the condition in the implication~\ref{eqn:second_ineq} is satisfied. In fact, $t$ equals one if and only if both $u$ and $v$ are equal to one. Furthermore, based on inequations $\circled{6}$, $\circled{7}$, $\circled{8}$ and $\circled{9}$, when $t = 1$, either both $y_1-b_1 \geq 0$ and $y_2-b_2 \geq 0$ (when $w=0$) or both $y_1-b_1 \leq 0 $ and  $y_2-b_2 \leq 0$ (when $w=1$). Conversely, when $t = 0$, the last four inequalities become trivial.
\end{proof}
\begin{lemma}
    \label{lemma:third}
    Every problem in the form of
    \begin{align}\label{eqn:third_ineq}
        & \text{minimize} \quad |x_1| + |x_2| + ... + |x_n|\\
       & \text{subject to } \quad \mathbf{A} \mathbf{x}\geq \mathbf{B}
    \end{align}
    can be transformed into a linear programming problem.
    
\end{lemma}

\begin{proof}
For each decision variable $x_i$, we introduce an auxiliary variable $t_i$ and define a linear programming formulation as follows:

    \begin{equation}
    \label{eqn:forth_ineq}
    \begin{aligned}
        & \text{minimize} \quad t_1 + t_2 + ... + t_n\\
       & \text{subject to } \quad \mathbf{A} \mathbf{x}\geq \mathbf{B}
       \\ & \qquad  \qquad  \quad \: t_i \geq x_i  \qquad  \qquad  \qquad \quad 1\leq i \leq n
       \\ & \qquad  \qquad  \quad \:  t_i \geq -x_i  \qquad  \qquad  \qquad \:  1\leq i \leq n
        \end{aligned} 
         \end{equation}
In fact, since we have $t_i \geq x_i$ and $t_i \geq -x_i $, we can infer that $t_i \geq |x_i|$. Additionally, in the optimal solution, $t_i$ equals $|x_i|$. Therefore, these two problems are equivalent.
\end{proof}

\subsection{Other Types of Rectangle Overlaps}
\label{sec:cases}
\textbf{Case 2.} \autoref{fig:merge}(b) shows the second case, where there is no vertex of the blue rectangle inside of the black rectangle but there is at least one vertex of the black rectangle inside of the blue rectangle. To make sure that this case would not happen we have the following statements:
\begin{align*}\nonumber
    \textbf{if} \:\: a'_1 + &\varepsilon_1^x + \varepsilon^x_c \leq a_1  \leq a'_2 + \varepsilon_2^x + \varepsilon^x_c \:\: \Rightarrow \\ 
    & \left( b_1 \geq b'_1 + \varepsilon^y_1 + \varepsilon^y_c \right) \: \text{or} \: \left( b_1 \leq b'_2 + \varepsilon^y_2  + \varepsilon^y_c \right),\\ \nonumber 
    \textbf{if} \:\: a'_1 + &\varepsilon_1^x + \varepsilon^x_c\leq a_1  \leq a'_2 + \varepsilon_2^x  \varepsilon^x_c \:\: \Rightarrow\\
    &\left( b_2 \geq b'_1 + \varepsilon^y_1 + \varepsilon^y_c \right) \: \text{or} \: \left( b_2 \leq b'_2 + \varepsilon^y_2  + \varepsilon^y_c\right),\\ \nonumber 
    \textbf{if} \:\: a'_1 + & \varepsilon_1^x + \varepsilon^x_c\leq a_2  \leq a'_2 + \varepsilon_2^x + \varepsilon^x_c \:\: \Rightarrow \\
    &\left( b_1 \geq b'_1 + \varepsilon^y_1 + \varepsilon^y_c \right) \: \text{or} \: \left( b_1 \leq b'_2 + \varepsilon^y_2  + \varepsilon^y_c\right),\\ \nonumber 
    \textbf{if} \:\: a'_1 + & \varepsilon_1^x + \varepsilon^x_c\leq a_2  \leq a'_2 + \varepsilon_2^x + \varepsilon^x_c\:\: \Rightarrow \\
    &\left( b_2 \geq b'_1 + \varepsilon^y_1 + \varepsilon^y_c \right) \: \text{or} \: \left( b_2 \leq b'_2 + \varepsilon^y_2 + \varepsilon^y_c  \right).
\end{align*}

Note that, each of these statements guarantees that a specific vertex of the blue rectangle is not inside of the black rectangle. The intuition is very similar to case 1, when we change the blue rectangle with the black rectangle. Constraints similar to other configurations are similar.

Now, based on Lemma~\ref{lemma:first}, each of these statements can be written as some linear inequalities in the form of mixed integer programming. We use $\mathcal{I}_2$ to refer to the set of all inequalities obtain from the above procedure for all candidate pair of rectangles.

\textbf{Case 3.}
\autoref{fig:merge}(c) shows the third case, where there is no vertex of the blue rectangle inside of the black rectangle and there is no vertex of the black rectangle inside of the blue rectangle but there is still an overlap. To make sure that this case would not happen we have the following statements:

\begin{align*}
    &\textbf{if} \:\: a_1 \leq a'_1 + \varepsilon^x_1 + \varepsilon^x_c \leq a_2  \:\: \Rightarrow \\ 
    &\left( b'_1 + \varepsilon^y_1 + \varepsilon^y_c - b_1 \right) 
     \left( b'_2 + \varepsilon^y_2 + \varepsilon^y_c - b_1 \right) \geq 0,\\
    &\textbf{if} \:\: a_1 \leq a'_1 + \varepsilon^x_1 + \varepsilon^x_c \leq a_2  \:\: \Rightarrow \\ 
    &\left( b'_1 + \varepsilon^y_1 + \varepsilon^y_c - b_2 \right) 
     \left( b'_2 + \varepsilon^y_2 + \varepsilon^y_c - b_2 \right) \geq 0,\\
    &\textbf{if} \:\: a_1 \leq a'_2 + \varepsilon^x_2 + \varepsilon^x_c \leq a_2  \:\: \Rightarrow \\ 
    &\left( b'_1 + \varepsilon^y_1 + \varepsilon^y_c - b_1 \right) 
     \left( b'_2 + \varepsilon^y_2 + \varepsilon^y_c - b_1 \right) \geq 0,\\
    &\textbf{if} \:\: a_1 \leq a'_2 + \varepsilon^x_2 + \varepsilon^x_c \leq a_2  \:\: \Rightarrow \\ 
    &\left( b'_1 + \varepsilon^y_1 + \varepsilon^y_c - b_2 \right) 
     \left( b'_2 + \varepsilon^y_2 + \varepsilon^y_c - b_2 \right) \geq 0,\\
    &\textbf{if} \:\: b_1 \leq b'_1 + \varepsilon^y_1 + \varepsilon^y_c \leq b_2  \:\: \Rightarrow \\ 
    &\left( a'_1 + \varepsilon^x_1 + \varepsilon^x_c  - a_1 \right) 
     \left( a'_2 + \varepsilon^x_2 + \varepsilon^x_c - a_1 \right) \geq 0,\\
    &\textbf{if} \:\: b_1 \leq b'_1 + \varepsilon^y_1 + \varepsilon^y_c \leq b_2  \:\: \Rightarrow \\ 
    &\left( a'_1 + \varepsilon^x_1 + \varepsilon^x_c - a_2 \right) 
     \left( a'_2 + \varepsilon^x_2 + \varepsilon^x_c - a_2 \right) \geq 0,\\
    &\textbf{if} \:\: b_1 \leq b'_2 + \varepsilon^y_2 + \varepsilon^y_c \leq b_2  \:\: \Rightarrow \\ 
    &\left( a'_1 + \varepsilon^x_1 + \varepsilon^x_c - a_1 \right) 
     \left( a'_2 + \varepsilon^x_2 + \varepsilon^x_c - a_1 \right) \geq 0,\\
    &\textbf{if} \:\: b_1 \leq b'_2 + \varepsilon^y_2 + \varepsilon^y_c \leq b_2  \:\: \Rightarrow \\ 
    &\left( a'_1 + \varepsilon^x_1 + \varepsilon^x_c  - a_2 \right) 
     \left( a'_2 + \varepsilon^x_2 + \varepsilon^x_c - a_2 \right) \geq 0.
\end{align*}

Note that, here, each of these statements guarantees that a specific side of the blue rectangle does not have an intersection with any of the sides of the black rectangle. For example, the first statement guarantees that two different sides of the blue rectangle cannot be at the different sides of the black rectangle. 

Now, based on Lemma~\ref{lemma:second}, each of these statements can be written as some linear inequalities in the form of mixed integer programming. We use $\mathcal{I}_3$ to refer to the set of all inequalities obtain from the above procedure for all candidate pair of rectangles.

\section{Additional Experiments}

\subsection{The Effect of the Grid and Buffer Width on Performance}
Next, we investigate the effect of grid and buffer width\footnote{The unit of the width depends on the scale of the map. Here, converting the unit to meters, $5 \times 10^{-4} \approx 100 m$.} on the performance of the \model. In section~\ref{ch:KG}, we used grids to capture the neighborhood of each building and used buffers to capture the buildings around each segment. Using inappropriate values for these parameters can damage the performance, causing poor conflated map. The reason is that using a large value as either grid or buffer width means considering irrelevant buildings as the neighbors of the object. In \autoref{fig:buffer-grid}(a), we report the average of the number of neighbors that we consider for each building when we change the width of the grid. As expected, using a small width for the buffer cannot capture the neighborhood as it only considers a small number of buildings as neighbors. On the other hand, using a large value for the grid width can result in considering a large number of buildings as neighbors, which misses the local property around each entity. We can see a similar pattern when we change the width of the buffer. Using a small value for the buffer width can cause missing all the neighbor buildings around each segment, missing the entire neighborhood. On the other hand, a large value for the buffer width can cause irrelevant buildings included in the neighborhood of each segment.

These results show how the width of the grid and buffer can affect the number of neighbors. However, its effect on the performance of our model remained unexplored. To this end, in \autoref{fig:buffer-grid1}, we report the F1 score of our model in the map matching task when we change the width of the buffer and grid. As expected, using large and small values for the width of the buffer and the grid can cause damaging the performance. Based on the scale of the data, there is an interval that the model achieves good performance. Comparing \autoref{fig:buffer-grid} and \autoref{fig:buffer-grid1}, we can see that in a specific interval, while the number of neighbors is changing, the performance of \model{} is not changing a lot. The main reason is the attention mechanism in knowledge graph encoding. The attention mechanism, based on the data, learns to emphasize relevant neighbors, which makes the model more robust to the value of the width of the grid and buffer.

Based on our experimental results, we realize that considering the value of the grid width close to the average of the building width in the map ($\approx 100m$) results in the best performance of the model. \autoref{fig:bufferreal} shows an example of a grid with width $= 100m$ surrounding a building in Bosten datasets.

\begin{figure}
    \centering
    \includegraphics[width=0.4\textwidth]{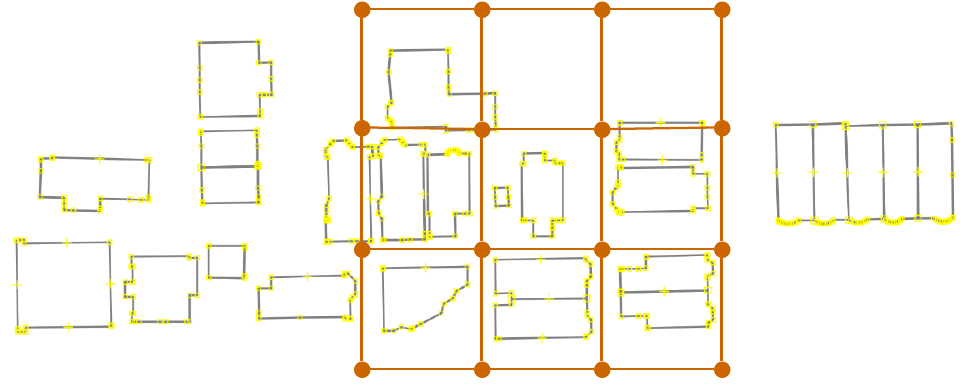}
    \caption{An example of a grid surrounding a building in Boston Open Data. In this example, all buildings located within the grid are considered as neighbors of the central building in the corresponding knowledge graph.}
    \label{fig:bufferreal}
    \end{figure}

\begin{figure}[ht]
    \centering
    \subfloat[][]{
        \includegraphics[width=0.22\textwidth]{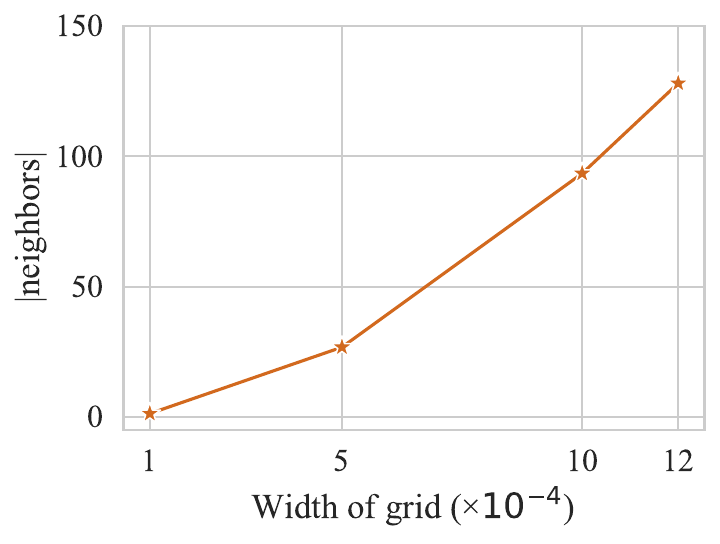}
    }
    \hfill
    \subfloat[][]{
        \includegraphics[width=0.22\textwidth]{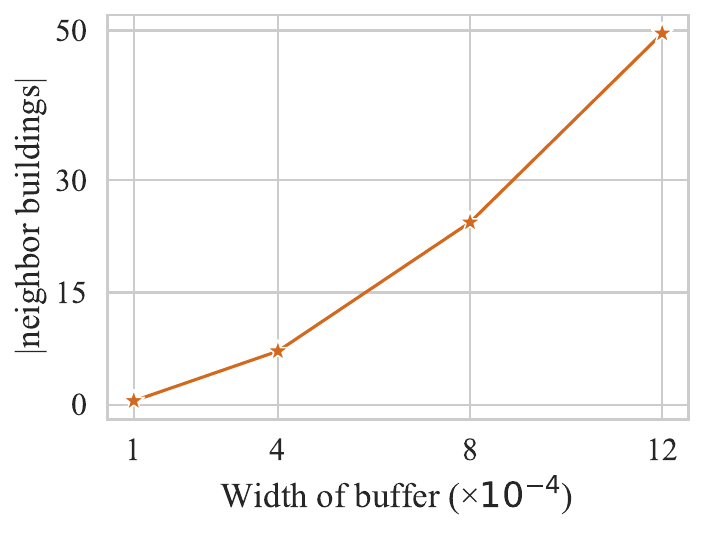}
    }
    \caption{(a) The average number of neighbors for different values of grid width. (b) The average number of neighboring buildings for different values of buffer width. As the width of buffers and grids increases, the number of neighbors for each entity also increases.}
    \label{fig:buffer-grid}
\end{figure}

\begin{figure}[ht]
    \centering
    \subfloat[][]{
        \includegraphics[width=0.22\textwidth]{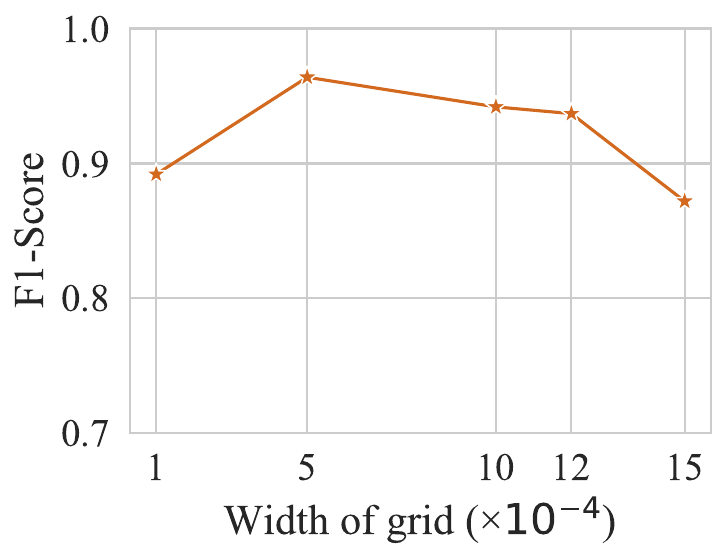}
    }
    \hfill
    \subfloat[][]{
        \includegraphics[width=0.22\textwidth]{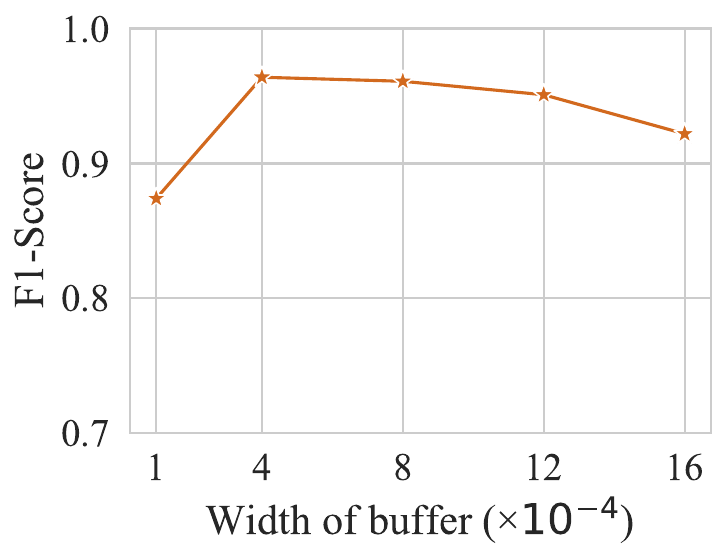}
    }
    \caption{(a) The F1-Score of matching for different values of grid width. (b) The F1-Score of matching for different values of buffer width.}
    \label{fig:buffer-grid1}
\end{figure}

\subsection{Ablation Study}
In \autoref{ch:map-matching}, we used several components to capture different aspects of the information provided in the map. Here, we evaluate the importance and contribution of each component to the performance of \model. To this end, each time, we remove one of the components, and keep other parts unchanged. \autoref{tab:ablation}
reports the results. The first row shows the performance of the \model. In the second row, we remove the relation loss and train our model only with the contrastive loss. Accordingly, our model cannot capture the semantics of relations. The modified model performance shows the importance of the relation-loss in the process of learning the map by knowledge graphs. The third row reports the performance of \model{} without encoding the 2-hop neighborhood, removing the multi-hop encoder module. \model{} outperforms this modified model by $1.67\%$ with respect to precision and $0.11\%$ with respect to recall. These results show the significant contribution of the multi-hop encoder in the performance of \model. Finally, we remove the knowledge graph encoding and only use area similarity. As discussed in previous sections, using Jaccard Similarity alone misses the structure around each entity, causing poor performance. All these results together show the importance of each of the components in the design of \model.

\begin{table}
 \caption{An ablation study showing the effectiveness of each component of \model{}. Four variants of \model{} were developed: the full model \model{}, \model{} (w\textbackslash o rel. loss) which excludes optimization of the relation loss, \model (w\textbackslash o 2-hop) that does not incorporate the 2-hop neighborhood, and a model that does not benefit from the knowledge graph encoder."}
 \resizebox{\columnwidth}{!} {
\begin{tabular}{ l | c | c| c| c|c  }
 \toprule 
{Model}
  & \multicolumn{1}{|p{2cm}|}{\centering Correct \\ Matches} &  \multicolumn{1}{|p{2cm}|}{\centering Incorrect \\ Matches} & \multicolumn{1}{|p{2cm}|}{\centering Missing \\ Matches} & {Precision} & {Recall}\\
 \midrule \midrule
 \model{} &  0.942  &  0.025  &  0.033  &  0.974 & 0.966\\
 \hline 
   \model{}
   w\textbackslash o rel. loss&  0.884  & 0.081  &  0.035  & 0.916  & 0.961 \\
   \hline
   \model{}
   w\textbackslash o 2hop     &   0.927  &  0.040  &  0.033  & 0.958  & 0.965\\
    \hline
  Area similarity &  0.803  &  0.084 &  0.113   & 0.905  & 0.876\\
 \bottomrule
\end{tabular}
}
\label{tab:ablation}
\end{table}

\subsection{Why Does Similarity Fail?}\label{subsec:fails}
As we discussed in section~\ref{ch:map-matching}, the similarity of the objects' area is an important factor in map matching. However, it is not sufficient to only look at it in the matching process and ignore the structure of the neighborhood of each object. In this part, we discuss, why similarity alone is not sufficient and show the cases that it fails. 

As discussed and shown in \autoref{tab:combined-matching}, \model{} outperforms Jaccard Similarity in sidewalk, building, and map (i.e., all entities)  matching tasks with significant improvement with respect to precision and recall. To interpret these results, we report the distribution of Jaccard similarity scores computed between
buildings (resp. sidewalks) in the OSM dataset and the BOD dataset in \autoref{fig:count} (resp. \autoref{fig:countseg}). These figures show that the Jaccard similarities of entities from two GDBs are not perfectly matched and there are $\sim30\%$ of all entities that do not have a Jaccard Similarity of more than $0.5$ with any other buildings or sidewalk.

\begin{figure}
\centering
{{\includegraphics[width=0.5\textwidth]{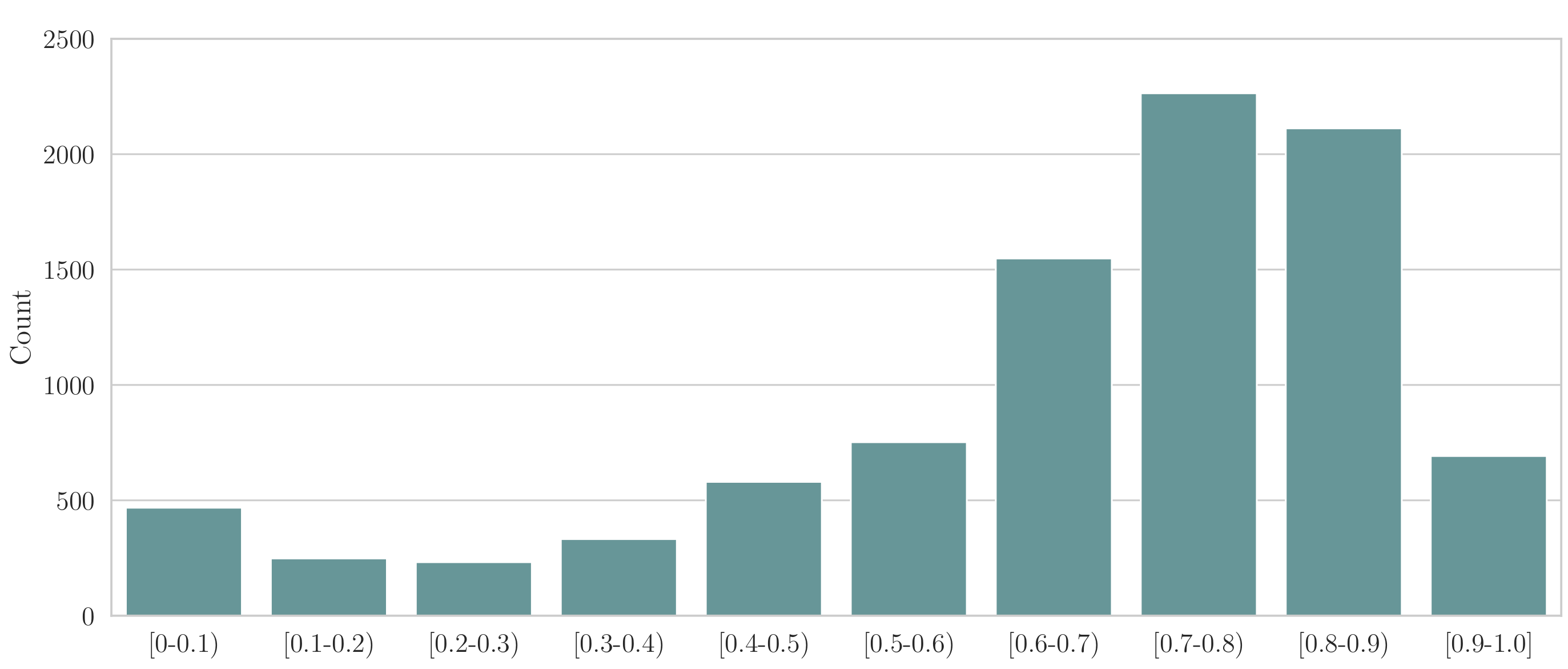} }}
    \caption{Distribution of Jaccard similarity scores computed between buildings in the OSM dataset and the BOD dataset. The Jaccard similarity scores represent the highest similarity between each building in the OSM data and buildings in the BOD dataset. This score indicates the degree of overlap between the buildings in the two datasets. }
    \label{fig:count}
\end{figure}

\begin{figure}
\centering
{{\includegraphics[width=0.5\textwidth]{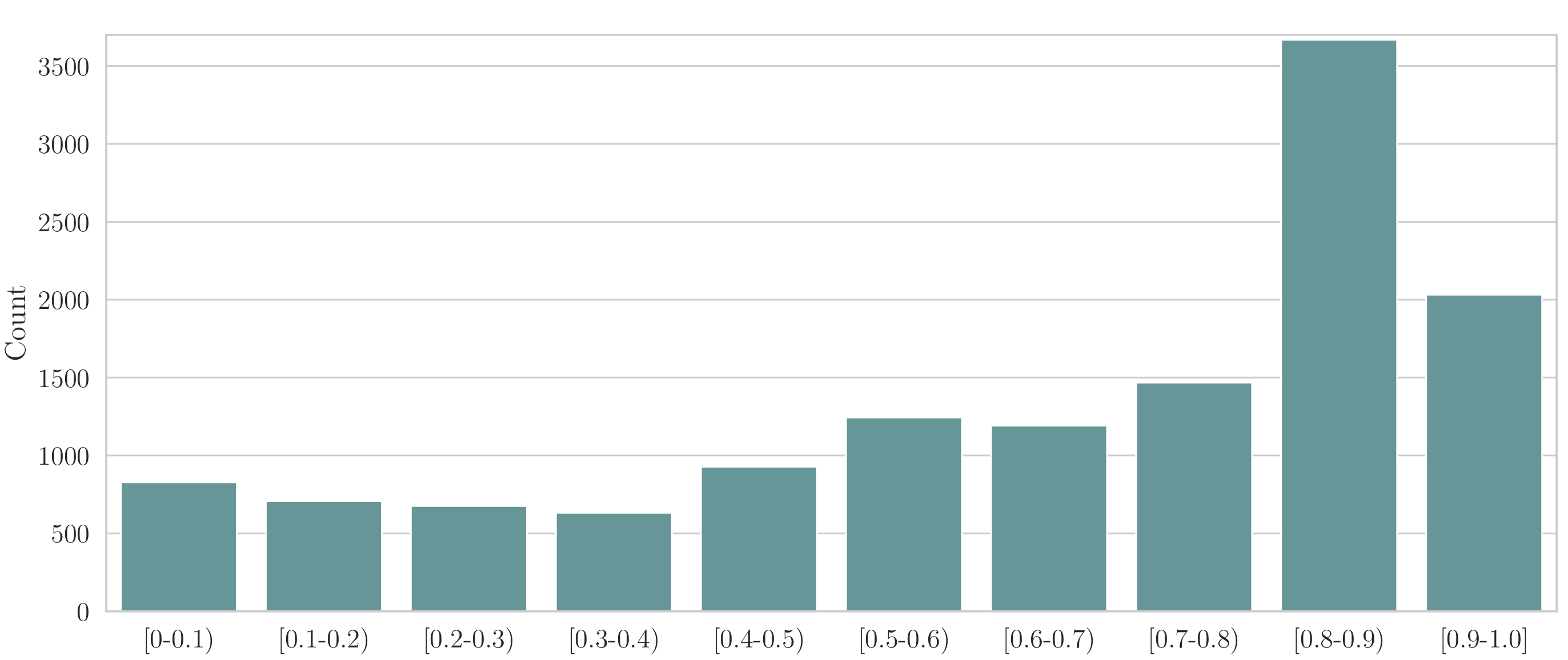} }}
    \caption{ Distribution of Jaccard similarity scores computed between buffer of segments in the OSM dataset and the BOD dataset. The Jaccard similarity scores represent the highest similarity between each segment in the OSM data and segments in the BOD dataset. }
    \label{fig:countseg}
\end{figure}

\autoref{fig:numbers} reports the distribution of the number of potential matches for each building. Most of the buildings have at least two potential matches, and there are $3\%$ of buildings that each has at least six potential matches.  These results show that due to the noise in different GDBs, the Jaccard Similarity alone is not sufficient to match entities, and we require to use structural properties around each entity to achieve a good performance, as \model{} does. 

\begin{figure}
\centering
{{\includegraphics[width=0.5\textwidth]{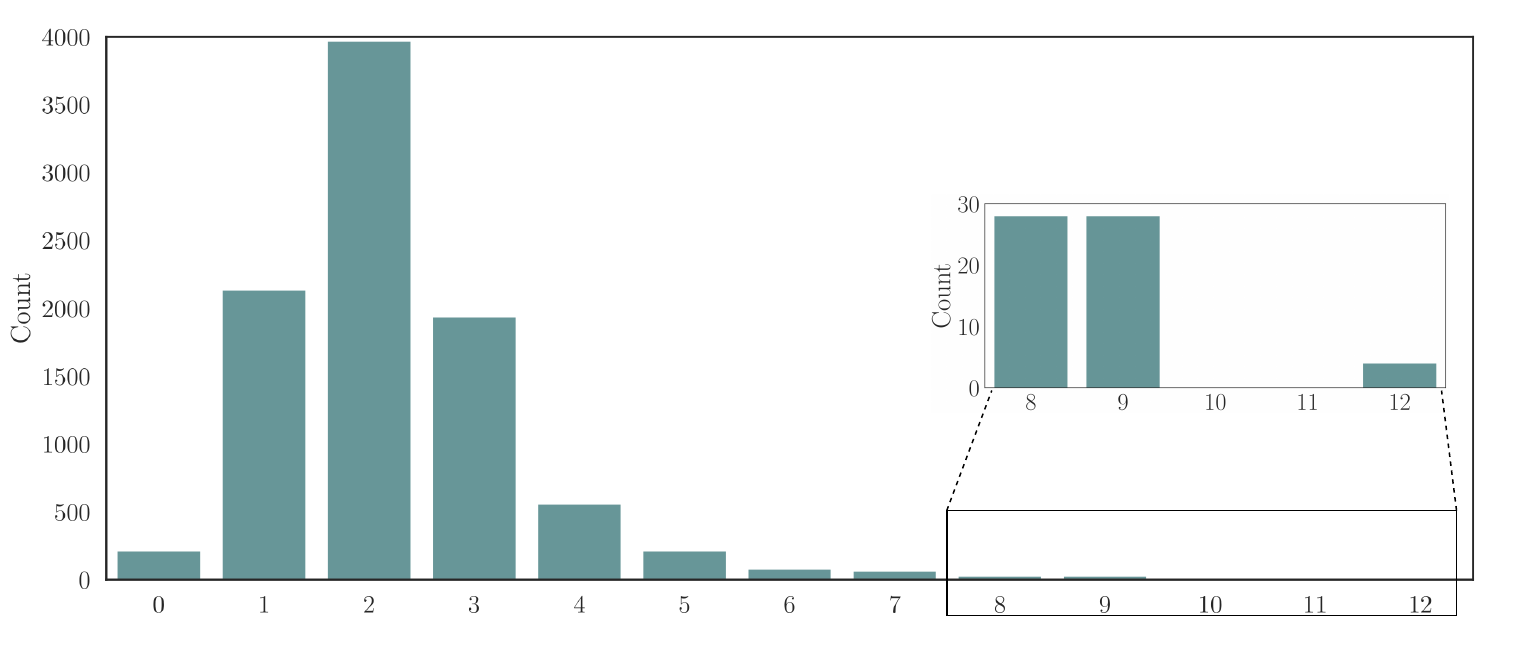} }}
    \caption{Distribution of the number of overlapping buildings in BOD for buildings in OSM.}
    \label{fig:numbers}
\end{figure}

\section{Limitations and Future Work}
In this section, we discuss some of the potential future work that might improve the designed framework. \circled{1} When constructing the knowledge graph representation of the map, a fixed grid and/or buffer size might cause false positive or false negative neighbors in the neighborhood of each entity. That is, in denser (resp. sparser) areas of the map, we might need to use smaller (resp. larger) values for the grid and buffer size. Moreover, the shape and size of the object might also affect the neighborhood. For example, larger objects require a larger grid or buffer size as they have more neighbors on the map. To this end, we need an adaptive approach that uses a dynamic grid and buffer size for each entity based on its situation and neighborhood. \circled{2} In different geospatial databases, the same object might have different representations. That is, a road can be represented using line segments in one GDB and using a polygon in the other. In the \model{} framework, we use different procedures for linear and non-linear entities. For example, we use square-shaped grids around each non-linear object, while buffers are used to capture the neighborhood of the linear objects. Also, in the knowledge graph construction, we use different relation types for linear and non-linear entities. One future direction is to adapt our approach so it can handle different types of representations of the same entity in different GDBs. \circled{3} In the map merging phase, we use minimum bounding rectangles for each object, including linear and non-linear entities. These rectangles' sides are parallel to the $\mathbf{x}$ and $\mathbf{y}$ axes, which can cause nonfunctioning space around objects that are not parallel to the$\mathbf{x}$ and $\mathbf{y}$ axises. To  this end, we can see a map at different levels of granularity and hierarchically match entiities. For example, we can divide each city into several distinct areas, and first match the source and target GDBs at the level of distinct areas. Then we can match entities within the matched areas. In this case, we can rotate the $\mathbf{x}$ and $\mathbf{y}$ axises in each distinct area such that the  nonfunctioning space for objects is minimized.


\end{document}